%% file: main.tex
\documentclass[12pt]{article}
\usepackage[letterpaper,left=1in,bottom=1in,top=1in,right=1in]{geometry}

\usepackage[utf8]{inputenc}
\usepackage[T1]{fontenc}
\usepackage[english]{babel}

\usepackage{microtype}
\usepackage{csquotes}

\usepackage{setspace}
\usepackage[en-GB]{datetime2}
\DTMlangsetup[en-GB]{showdayofmonth=false}

\usepackage{titlesec}

\titleformat{\part}[frame]
{\normalfont}
{\filright \footnotesize \enspace PART \thepart \enspace}
{8pt}
{\vspace{20pt} \huge \bfseries \filcenter}

\titlespacing*{\part}{0pt}{-50pt}{10pt}

\usepackage{titletoc}

\usepackage{enumitem}

\usepackage{graphicx}
\usepackage{tikz}
\usepackage{caption}
\usepackage{subcaption}
\usepackage{float}

\usepackage{xcolor}

\definecolor{linkred}{RGB}{125,25,50}

\input{math_commands.tex}

\usepackage[numbers]{natbib}
\usepackage{hyperref}
\hypersetup{
    colorlinks=true,
    linkcolor=linkred,
    citecolor=linkred,
    urlcolor=linkred,
    linktocpage=true,
    breaklinks=true,
    pdftitle={Adversarially Robust PAC Learning}
}

\usepackage{xurl}

\usepackage[capitalize,noabbrev]{cleveref}

\newenvironment{keywords}{\vspace{-\baselineskip} \leftskip 30pt \rightskip 30pt \noindent \ignorespaces \small {\bfseries Keywords:}}

\title{Adversarially Robust PAC Learning with Optimal VC Rates}

\author{
    Steve Hanneke\thanks{Department of Computer Sciences, Purdue University, USA. Email: \texttt{steve.hanneke@gmail.com}} \and
    Amirreza Shaeiri\thanks{Department of Computer Sciences, Purdue University, USA. Email: \texttt{amirreza.shaeiri@gmail.com}}
}

\date{\today}

\begin{document}

\maketitle

\thispagestyle{empty}

\begin{abstract}
We study the problem of \emph{adversarially robust} PAC learning. In this framework, the learner observes independent samples from an unknown distribution over $\mathcal{X} \times \{0,1\}$, as in classical PAC learning. However, given a perturbation map $\mathcal{U} : \mathcal{X} \to 2^{\mathcal{X}}$ known to the learner, the goal is to output, with high probability, a predictor that correctly classifies \emph{every} perturbation $z \in \mathcal{U}(x)$ of most future examples $(x,y)$ drawn from the same underlying distribution.

We determine the \emph{optimal} $\mathcal{U}$-independent sample complexity of this problem in both the realizable and agnostic settings. More specifically, for every concept class $\mathcal{H}$ of $\operatorname{VC}$ dimension $d$, we prove upper bounds of $\mathcal{O} \big( d/\epsilon + \log(1/\delta)/\epsilon \big)$ in the realizable setting and $\mathcal{O} \big( d/\epsilon^2 + \log(1/\delta)/\epsilon^2 \big)$ in the agnostic setting, together with an optimal first-order refinement of the latter. These bounds match the corresponding lower bounds for classical PAC learning. Consequently, and perhaps surprisingly, adversarial robustness incurs \emph{no additional} distribution-free statistical cost, uniformly over all perturbation maps. Our bounds improve exponentially on those of [Montasser, Hanneke, and Srebro; COLT '19].


On the technical side, we present short and elementary proofs based on a new algorithmic principle that we call \emph{binomial-bagging}. We believe that binomial-bagging and its analysis may be of independent interest.\\
\end{abstract}


\clearpage

\setcounter{page}{1}
\setcounter{tocdepth}{2}\tableofcontents

\clearpage

\input{Main/Introduction}

\input{Main/Preliminaries}

\input{Main/Proofs}

\input{Main/Acknowledgments}

\section*{AI Disclosure}
The main technical ideas underlying this work emerged through an extensive iterative process involving the authors, ChatGPT 5.6 Sol Pro, and Fable 5. The manuscript was subsequently written by the authors with additional assistance from ChatGPT 5.6 Sol and Claude Opus 5. The authors take full responsibility for all content in the paper.

\clearpage

\bibliography{References}
\addcontentsline{toc}{section}{References}

\end{document}

%% file: math_commands.tex
\usepackage{amsmath}
\allowdisplaybreaks[4]
\usepackage{amssymb}
\usepackage{mathtools}
\usepackage{amsthm}
\usepackage{thmtools}
\usepackage{etoolbox}

\usepackage{amsfonts}
\usepackage{mathrsfs}
\usepackage{dsfont}
\usepackage{bm}

\usepackage{relsize}
\usepackage{nicefrac}
\usepackage{centernot}

\usepackage{accents}

\usepackage{algorithm}

\usepackage{algpseudocode}

\theoremstyle{plain}
\newtheorem{theorem}{Theorem}
\newtheorem{lemma}[theorem]{Lemma}

\theoremstyle{definition}
\newtheorem{definition}[theorem]{Definition}

\theoremstyle{remark}

\AtBeginEnvironment{proof}{\setcounter{localclaim}{0}}



\DeclareMathOperator*{\argmin}{arg\,min}

\usepackage{tcolorbox}

\tcbset{
    rounded corners,
    colback = white,
    before skip = 0.25cm,
    after skip = 0.5cm,
    boxrule = 2pt,
    arc = 10pt
}

\usepackage{tikz}

\newcommand{\circled}[1]{
  \tikz[baseline=(char.base)]{
    \node[
      shape=circle,
      draw,
      inner sep=0.4pt,
      minimum size=1.2em
    ] (char) {\strut #1};
  }
}

\usepackage[framemethod=tikz]{mdframed}

\newmdenv[
  linewidth=0.6pt,
  linecolor=black!65,
  roundcorner=2pt,
  skipabove=12pt,
  skipbelow=12pt,
  innerleftmargin=9pt,
  innerrightmargin=9pt,
  innertopmargin=7pt,
  innerbottommargin=8pt
]{theoremframe}



%% file: Main/Introduction.tex
\section{Introduction} \label{sec:introduction}

We prove that adversarial robustness incurs no additional distribution-free statistical cost: even in the worst case over perturbation maps, the optimal sample complexities of adversarially robust PAC learning in both the realizable and agnostic settings match those of ordinary PAC learning, up to universal constant factors.

Notably, concurrent and independent work by \cite{montasser2026baggingrobustlylearnsvc} obtained the same result in the realizable setting slightly earlier, albeit using a different technique. We refer the reader to \autoref{sec:concurrent-work} for a comprehensive discussion.

\paragraph{Background.} The phenomenon of test-time robustness of machine learning models has been studied for more than two decades \cite{dalvi2004adversarial, lowd2005adversarial, globerson2006nightmare, kolcz2009feature}. However, the seminal works of \cite{szegedy2013intriguing, goodfellow2015explaining}, which studied this topic for image classification under the name of \emph{adversarial examples}, can be considered the beginning of significant attention to this field, especially in the context of deep learning.~\footnote{See the figure at the following \href{https://nicholas.carlini.com/writing/2019/all-adversarial-example-papers.html}{link} for an illustration of the approximate cumulative number of papers published on related topics over time.} Since then, \emph{adversarial robustness} has been studied in a wide range of domains, including voice recognition \cite{carlini2016hidden, carlini2018audio} and reinforcement learning \cite{lin2017tactics, tretschk2018sequential}. More recently, with the emergence of large language models, adversarial attacks have taken new forms, including adversarial prompts and jailbreak attacks \cite{perez2022ignore, greshake2023not, wei2023jailbroken, zou2023universal}.

Substantial effort has been devoted to defending models against adversarial attacks, including various heuristic and provable defenses. Many proposed heuristic defenses were subsequently broken by stronger or adaptive attacks \cite{carlini2016defensive, carlini2017magnet, carlini2017evaluating, athalye2018robustness, athalye2018obfuscated, carlini2019ami, carlini2019evaluating, tramer2020adaptive}, while provable defenses have historically faced substantial challenges in scaling to sufficiently expressive networks and large-scale datasets such as ImageNet \cite{pmlr-v97-cohen19c}. Among existing approaches, \emph{adversarial training} \cite{goodfellow2015explaining, kurakin2016adversarial, madry2017towards} has arguably achieved the greatest practical success. In addition, adversarial robustness appears with a variety of unexpected benefits, significantly increasing model interpretability \cite{tsipras2019robustness, zhang2019interpreting} \cite{kaur2019perceptuallyaligned}, better model transferability \cite{salman2020adversarially}, and utility in a diverse range of computer vision tasks \cite{santurkar2019image, engstrom2019adversarial, salehi2020arae}.~\footnote{The references cited in the first two paragraphs of the introduction are intended to provide a selective, rather than comprehensive, account of the literature.}

\paragraph{Formal framework.} The focus of this work is the formal study of adversarial robustness within the framework of \emph{probably approximately correct (PAC) learning}. The seminal works of \cite{vapnik1974theory} and \cite{valiant1984theory} laid the foundations of modern statistical learning theory by introducing this framework, which has become a central theoretical criterion for supervised learning and has been the focus of extensive research over the past several decades.

We next present the fundamental framework of adversarially robust PAC learning studied in this work. In this framework, there is an unknown data distribution $\mathcal{D}$ over $\mathcal{X} \times \mathcal{Y}$, where $\mathcal{X}$ is an instance space and $\mathcal{Y}=\{0,1\}$ is a binary label space.~\footnote{For clarity of exposition, we omit standard set and measure theoretic qualifications in the introduction. Formal definitions are provided in \autoref{sec:preliminaries}.} For example, $\mathcal{X}$ may be a space of medical images, with $y=1$ indicating that the corresponding image shows a certain abnormality and $y=0$ indicating that it does not. The framework also fixes a non-empty concept class $\mathcal{H}$ consisting of functions from $\mathcal{X}$ to $\mathcal{Y}$, which serves as a benchmark family of classifiers. For instance, $\mathcal{H}$ may be the class of neural networks with a prescribed architecture.

Adversarial robustness is modeled by a known perturbation map $\mathcal{U} : \mathcal{X} \to 2^{\mathcal{X}}$, where $\mathcal{U}(x)$ is the set of allowable perturbations of an instance $x$. For image classification, a standard choice is $\mathcal{U}(x) = \{ z \in \mathcal{X} \; | \; \lVert z-x \rVert_{\infty} \leq \rho \}$, where $\rho > 0$ is a small radius. This permits each pixel value to change by at most $\rho$ and is commonly used to obtain visually imperceptible perturbations of an image. Based on that, for a classifier $f : \mathcal{X} \to \{0,1\}$ and an example $Z = (x,y)$, write $\ell_{\mathcal{U}}(f;Z) := \mathds{1} \bigl\{ \exists{z \in \mathcal{U}(x)} \; \text{such that} \; f(z) \neq y \bigr\}$. Now, given $m$ independent samples drawn from $\mathcal{D}$, the learner’s goal is to output a hypothesis $\hat{h} : \mathcal{X} \to \mathcal{Y}$ whose robust error $R_{\mathcal{U},\mathcal{D}}(\hat{h}) := \mathbb{E}_{Z \sim \mathcal{D}} \, \bigl[\ell_{\mathcal{U}}(\widehat{h};Z)\bigr]$, with high probability over the observed samples, is close to that of the best concept in $\mathcal{H}$, namely $R_{\mathcal{U},\mathcal{D}}^\star(\mathcal{H}) := \inf_{h \in \mathcal{H}} R_{\mathcal{U},\mathcal{D}}(h)$. Moreover, we study both the \emph{realizable setting}, where $R_{\mathcal{U},\mathcal{D}}^\star(\mathcal{H}) = 0$, and the \emph{agnostic setting}, where no such assumption is imposed.

\paragraph{The central question.} This leads to the following natural and foundational question:

\begin{center}
    \textit{Given a concept class $\mathcal{H} \subseteq \{0,1\}^{\mathcal{X}}$, how many samples are \textbf{necessary} and \textbf{sufficient} for adversarially robust PAC learning of $\mathcal{H}$, uniformly over all perturbation maps $\mathcal{U}$?}
\end{center}

\paragraph{Prior work.} This question was first studied by \cite{pmlr-v99-montasser19a}. Although the empirical risk minimizer (ERM) yields a proper learning rule with essentially optimal sample complexity in classical PAC learning, they showed that its natural robust analogue, the robust empirical risk minimizer (RERM), can fail even in the realizable setting. Notably, RERM provides a formalization of adversarial training. More strongly, they exhibited a class of $\operatorname{VC}$ dimension one and a perturbation map under which no proper learner—that is, no learner whose output is constrained to lie in the class—can adversarially robustly PAC learn. Nevertheless, they proved that every $\operatorname{VC}$ class is adversarially robustly PAC learnable under every perturbation map via an improper boosting-based learner. For a class of $\operatorname{VC}$ dimension $d$, their bounds, up to logarithmic factors, depend on the class complexity through $d\cdot d^\star$, where $d^\star$, the dual $\operatorname{VC}$ dimension of the class, can be exponential in $d$ \cite{assouad1983densite, kleer2023primal}. Several subsequent works further investigated this problem. Most notably, \cite{pmlr-v151-montasser22a, montasser2022adversarially} introduced the global one-inclusion graph—an extension of the one-inclusion framework that underlies many recent developments in the theory of PAC learning—and an associated combinatorial complexity parameter, obtaining a minimax-optimal adversarially robust PAC learner for each pair $(\mathcal{H}, \mathcal{U})$. They conjectured that this complexity measure can be upper bounded by $\operatorname{VC}(\mathcal{H})$, uniformly over all perturbation maps $\mathcal{U}$. We also refer interested readers to the works of \cite{pmlr-v134-montasser21a, montasser2020reducing, attias2019improved, attias2022characterization, attias2023adversarially, pmlr-v272-attias25a}, which are closely related to the broader themes studied here, although their connection to our results is indirect.

\paragraph{A distinction from classical learning.} The above failures indicate that adversarially robust PAC learning cannot be understood by merely transplanting the classical theory; rather, it calls for genuinely new theoretical ideas. Moreover, empirical findings reinforce this message: smooth activation functions such as Swish/SiLU can outperform the widely used ReLU activation function under adversarial training \cite{xie2020smooth, gowal2020uncovering}, and continued adversarial training can decrease the robust training loss while worsening robust generalization, a phenomenon known as \emph{robust overfitting} \cite{rice2020overfitting}.

\paragraph{Our contributions.} We prove \emph{optimal} upper bounds on the $\mathcal{U}$-independent sample complexities of this problem in both the realizable and agnostic settings. Formally, we establish the following theorem.

\begin{theorem}[Main result] \label{theorem:introduction_1}
Let $\mathcal{H} \subseteq \{0,1\}^{\mathcal{X}}$ be a concept class with $d := \operatorname{VC}(\mathcal{H}) < \infty$, and let $\mathcal{U} : \mathcal{X} \to 2^{\mathcal{X}}$ be any perturbation map. Then, there exist a universal constant $C>0$ and an improper learning algorithm such that, for every distribution $\mathcal{D}$ over $\mathcal{X} \times \{0,1\}$, every sample size $m \geq 1$, and every $\delta \in (0,1)$, given an i.i.d.\ sample $S \sim \mathcal{D}^{m}$, the learner outputs a classifier $\widehat{h} : \mathcal{X} \to \{0,1\}$ satisfying, with a probability of at least $1-\delta$,
\[
    R_{\mathcal{U},\mathcal{D}}(\widehat{h}) \leq R_{\mathcal{U},\mathcal{D}}^\star(\mathcal{H}) + C \left( \sqrt{\frac{R_{\mathcal{U},\mathcal{D}}^\star(\mathcal{H}) \bigl(d + \log(1/\delta) \bigr)}{m}} + \frac{d + \log(1/\delta)}{m} \right).
\]
\end{theorem}

Next, we show two corollaries of the above theorem. In particular, in the realizable setting, for every $\epsilon, \delta \in (0,1)$,
\[
    m \in \mathcal{O} \left( \frac{d + \log(1/\delta)}{\epsilon} \right)
\]
i.i.d.\ samples suffice for an improper learner to output, with a probability of at least $1-\delta$, a classifier of robust error of at most $\epsilon$, uniformly over all perturbation maps $\mathcal{U}$. More generally, in the agnostic setting, for every $\epsilon, \delta \in (0,1)$,
\[
    m \in \mathcal{O} \left( \frac{d + \log(1/\delta)}{\epsilon^2} \right)
\]
i.i.d.\ samples suffice for an improper learner to output, with a probability of at least $1-\delta$, a classifier of robust error of at most $R_{\mathcal{U},\mathcal{D}}^\star(\mathcal{H}) +\epsilon$, uniformly over all perturbation maps $\mathcal{U}$.

Notably, all three bounds are optimal up to universal constant factors, as taking the identity perturbation map $\mathcal{U}(x) = \{x\}$ reduces adversarially robust PAC learning to its classical counterpart, for which matching lower bounds are known (for non-trivial concept classes).


Remarkably, all of our results are obtained from a simple algorithmic principle inspired by the classical method of bagging, which we call \emph{binomial bagging}. This improper learning principle trains robust empirical risk minimizers on sub-samples formed by retaining each training example independently with probability $1/2$, and aggregates the resulting predictions pointwise. For the expectation guarantees, we threshold the binomial-bagging score at $1/2$ in the realizable setting, whereas in the agnostic setting, we learn the threshold using an independent validation sample. To obtain high-probability guarantees, we apply binomial bagging to a range of sample prefixes: in the realizable setting, we threshold each prefix score at $1/2$ and take a majority vote over the resulting classifiers; in the agnostic setting, we instead average the prefix scores and learn a single threshold again using an independent validation sample. Furthermore, the corresponding proofs are surprisingly short and elementary. We believe that both the binomial-bagging principle and its analysis may be of independent interest. We refer the reader to \autoref{sec:implications} for further discussion.

\paragraph{Broader impact.} Perhaps surprisingly, our results show that adversarial robustness incurs no additional distribution-free statistical cost, uniformly across all perturbation maps. This is particularly striking given that an additional cost has been established in a distribution-specific framework of \cite{schmidt2018adversarially}, and substantial robust generalization gaps have been observed empirically for adversarial training \cite{tsipras2019robustness, rice2020overfitting}. In light of the fact that our algorithms use robust empirical risk minimizers—the formalization of adversarial training—entirely as a black box, a particularly compelling direction for future work is to translate the insights from our work into variants of adversarial training with improved robust generalization.

Additionally, one possible interpretation of the binomial bagging principle is as a form of \emph{overparameterization}: instead of enlarging a single model, it aggregates RERMs trained on different sub-samples. This echoes laws showing that additional parameterization can be necessary for smooth robust interpolation \cite{bubeck2021law, bubeck2021universal}, and recent work connecting these laws to robust generalization \cite{more2026order}. It also resonates with the equivalence, established for overparameterized random-feature models, between ensembles and single wider models \cite{dern2025theoretical}. Although no formal equivalence follows, these connections suggest that aggregation and overparameterization may be related ways of supplying effective capacity for robust generalization.

\subsection{Technical Overview} \label{sec:technical-overview}

In this subsection, we convey the main ideas of our approach by presenting a direct and complete proof of the optimal realizable guarantee. We then discuss the argument and explain how the same algorithmic principle yields the optimal first-order agnostic guarantee in \autoref{theorem:introduction_1}. For simplicity, we assume that the infimum in the definition of the realizable setting is attained; the same argument extends easily to the more general case.

\paragraph{Robust empirical risk minimization.} A robust empirical risk minimizer (RERM) is a mapping that assigns to every finite multiset of instance-label pairs $S$ a hypothesis $\hat{h} \in \mathcal{H}$ satisfying
\begin{equation*}
    \hat{h} \in \argmin_{h \in \mathcal{H}} \sum_{Z \in S} \ell_{\mathcal{U}}(h;Z),
\end{equation*}
where the sum is taken with multiplicity. We fix a deterministic RERM rule whose input is a finite multiset, including the empty multiset.

\begin{figure}[!htb]
\captionsetup{skip=5pt}
\centering
\begin{tcolorbox}

\begin{center}
\textbf{Exact Binomial Bagging of Robust ERMs}
\end{center}

\noindent \textbf{Input:}
A sample
\[
    S = (Z_1, Z_2, \ldots, Z_m) \in (\mathcal{X} \times \mathcal{Y})^m.
\]

\begin{enumerate}
    \item For every index set $A \subseteq [m]$, let
    \[
        S_A := (Z_i)_{i \in A} \qquad \text{and} \qquad h_A := \mathsf{RERM}(S_A).
    \]

    \item Define the binomial-bagging score
    \[
        b_S(x) := \frac{1}{2^m} \sum_{A \subseteq [m]} h_A(x), \qquad x \in \mathcal{X},
    \]
    \item Return the classifier $\operatorname{BB}_S : \mathcal{X} \rightarrow \mathcal{Y}$ such that for every $x \in \mathcal{X}$, we have
    \[
        \operatorname{BB}_S(x) := \mathds{1} \{ b_S(x) \geq 1/2 \}.
    \]
\end{enumerate}

\end{tcolorbox}
\caption{Exact binomial bagging. Equivalently, $A$ is generated by retaining each sample occurrence independently with probability $1/2$, and the expectation over this random mask is evaluated exactly.} \label{alg:binomial-bagging}
\end{figure}

\paragraph{Leave-one-out errors.} Fix a robustly realizable sample
\begin{equation*}
    S = \big( (x_1, y_1), (x_2, y_2), \ldots, (x_m, y_m) \big) \in (\mathcal{X} \times \mathcal{Y})^{m},
\end{equation*}
of size $m$ for some $m \in \mathbb{N}$, and for every $i \in [m]$, let $S^{-i}$ denote the multiset obtained by deleting the $i$-th occurrence. Define the set of leave-one-out robust mistakes
\begin{equation*}
    E(S) := \big\{ i \in [m] \; \big| \; \ell_{\mathcal{U}} \big( \operatorname{BB}_{S^{-i}}; (x_i,y_i) \big) = 1 \big\},
\end{equation*}
and let $q := |E(S)|$.

Next, for the remainder of the argument, for each $i \in E(S)$, we fix a witness $z_i \in \mathcal{U}(x_i)$ certifying the error, so that
\begin{equation*}
    \operatorname{BB}_{S^{-i}}(z_i) \neq y_i.
\end{equation*}
Indeed, such a $z_i$ must exist; otherwise, $\operatorname{BB}_{S^{-i}}(z) = y_i$ for every $z \in \mathcal{U}(x_i)$, contradicting $i \in E(S)$.

\begin{lemma}[Realizable leave-one-out bound]\label{lem:realizable-loo}
Let $d:=\operatorname{VC}(\mathcal H)<\infty$. For every robustly
realizable sample $S$ as above, binomial bagging satisfies
\begin{equation*}
    |E(S)|
    =\sum_{i=1}^m\ell_{\mathcal U}
       \bigl(\operatorname{BB}_{S^{-i}};(x_i,y_i)\bigr)
    \leq60d.
\end{equation*}
\end{lemma}

\begin{proof}
The case $q=0$ is immediate, so assume $q>0$.

\paragraph{Proof roadmap.} The heart of the proof consists of two counting steps. In each step, we introduce one quantity and compare a lower bound with an upper bound on it. The second comparison yields $q\leq60d$.

\paragraph{Step 1: many subsets make many errors.} For every $A \subseteq [m]$, define
\begin{equation*}
    e(A) := \left| \left\{ i \in E(S) \mid h_A(z_i) \neq y_i \right\} \right|,
\end{equation*}
and let
\begin{equation*}
    T := \sum_{A \subseteq [m]} e(A)
\end{equation*}
be the total number of error incidences $(A,i)$.

\circled{1} Fix $i \in E(S)$. Thus, we know that $\ell_{\mathcal{U}} \big( \operatorname{BB}_{S^{-i}}; (x_i,y_i) \big) = 1$. The $2^{m-1}$ subsets of $[m]$ omitting $i$ are exactly the subsets used to define $\operatorname{BB}(S^{-i})$. Since their majority is wrong at $z_i$, at least half of their RERMs are wrong at $z_i$. Consequently,
\begin{equation*}
    \big| \big\{ A \subseteq [m] \; \big| \; h_A(z_i) \neq y_i \big\} \big| \geq 2^{m-2}.
\end{equation*}
As a result, each of the $q$ hard coordinates contributes at least $2^{m-2}$ incidences. Therefore, we have
\begin{equation}
    T \geq q2^{m-2} = \frac{q}{4} \, 2^m. \label{eq:T-lower}
\end{equation}

\circled{2} Now, define
\begin{equation*}
    \mathcal{G} := \big\{ A \subseteq [m] \; \big| \; e(A) \geq \frac{q}{5} \big\}.
\end{equation*}
Since $e(A) \leq q$ for every $A$, while $e(A) < q/5$ for $A \notin \mathcal{G}$, we have
\begin{equation}
    T \leq q \, |\mathcal{G}| + \frac{q}{5} \, \bigl( 2^m - |\mathcal{G}| \bigr) = \frac{q}{5} \, 2^m + \frac{4q}{5} \, |\mathcal G|. \label{eq:T-upper}
\end{equation}

\circled{3} Combining \eqref{eq:T-lower} and \eqref{eq:T-upper} yields
\begin{equation*}
    \frac{q}{4} \, 2^m \leq \frac{q}{5} \, 2^m + \frac{4q}{5} \, |\mathcal{G}|,
\end{equation*}
and hence
\begin{equation}
    |\mathcal{G}| \geq \frac{2^m}{16}. \label{eq:G-large}
\end{equation}

\paragraph{Step 2: one large error pattern cannot occur too often.} Consider the binary error trace class
\begin{equation*}
    \mathcal{V} := \Big\{ \big( \mathds{1} \{ h(z_i) \neq y_i \} \big)_{i \in E(S)} \; \big| \; h \in \mathcal{H} \Big\} \subseteq \{0,1\}^{q}.
\end{equation*}
The fixed labels $y_i$ only flip the coordinates of the ordinary trace of $\mathcal{H}$, and repeated witnesses cannot increase the shattering dimension. Therefore, $\operatorname{VC}(\mathcal{V}) \leq d$. Thus, by the Sauer--Shelah--Perles lemma \cite{sauer1972density, shelah1972combinatorial}, we get
\begin{equation}
    M := |\mathcal{V}| \leq \sum_{j = 0}^{\min\{d,q\}} \binom{q}{j} \label{eq:number-traces}
\end{equation}

\circled{1} Now, every $A \in \mathcal{G}$ produces one pattern in $\mathcal{V}$. Thus, by \eqref{eq:G-large} and the pigeonhole principle, some fixed pattern $v \in \mathcal{V}$ is produced by at least
\begin{equation}
    F \geq \frac{2^m}{16M} \label{eq:F-lower}
\end{equation}
different subsets $A$.

\circled{2} Since this pattern occurs for a member of $\mathcal{G}$, its number of errors
\begin{equation*}
    r := \sum_{i \in E(S)} v_i
\end{equation*}
satisfies $r \geq q/5$. On the other hand, if $v_i = 1$ and a subset $A$ produces the pattern $v$, then $i \notin A$. Indeed, retaining $(x_i, y_i)$ would force $h_A(z_i) = y_i$ through robust consistency. Thus, every subset producing $v$ must omit the same $r$ specified coordinates. There are at most
\begin{equation}
    F \leq 2^{m-r} \leq 2^{m-q/5} \label{eq:F-upper}
\end{equation}
such subsets.

\circled{3} Comparing \eqref{eq:F-lower} and \eqref{eq:F-upper} gives
\begin{equation*}
    \frac{2^m}{16M} \leq 2^{m-q/5}, 
\end{equation*}
or equivalently,
\begin{equation}
    2^{q/5} \leq 16M. \label{eq:key-counting-inequality}
\end{equation}

\paragraph{Conclusion.} If $d = 0$, all hypotheses in a non-empty binary class coincide, and robust realizability immediately gives $q = 0$. Suppose, therefore, that $d \geq 1$. If $q \leq d$, there is nothing to prove. Otherwise, the standard numerical form of Sauer's lemma and \eqref{eq:key-counting-inequality} give
\begin{equation*}
    2^{q/5} \leq 16 \, \left(\frac{eq}{d}\right)^d.
\end{equation*}
Now, writing $x=q/d$, this inequality cannot hold for $x \geq 60$: the function
\begin{equation*}
    x \longmapsto \frac{2^{x/5}}{ex} 
\end{equation*}
is increasing for $x \geq 60$, and at $x = 60$ it is already larger than $16$. Hence, $q\leq60d$, as claimed.
\end{proof}

\paragraph{Alternative completions of the argument.} The same retained-versus-omitted dichotomy also admits alternative completions using Rademacher complexity or mutual information. Draw a subset $A \subseteq [m]$ uniformly at random. For every hard coordinate $i$, retaining $i$ forces $h_A$ to be correct at the fixed witness $z_i$, whereas, conditional on omitting $i$, it is wrong there with probability at least $1/2$. These properties force a constant positive correlation between each error bit and the random sign encoding its omission. A standard Rademacher bound controls the sum of these correlations and yields $q \in \mathcal{O}(d)$. Alternatively, each error reveals an omitted coordinate, giving a mutual-information lower bound proportional to $q$ between the omission pattern and the error trace. Sauer's lemma bounds the number of possible error traces and hence their entropy, again yielding $q \in \mathcal{O}(d)$. The two-step counting argument above provides a more direct and elementary route to the same conclusion.

\paragraph{From leave-one-out to high probability.} The binomial-bagging rule depends only on the input multiset and is therefore symmetric. As a result, for $d \geq 1$, we apply the prefix-averaging argument of \cite{aden2023optimal} to Lemma~\ref{lem:realizable-loo} to obtain the high-probability guarantee.~\footnote{The martingale argument underlying \cite[Theorem~2.1]{aden2023optimal} applies unchanged here because the robust loss takes values in $[0,1]$ and the leave-one-out bound above is deterministic.} Specifically, write

\begin{equation*}
    S_{\leq t} := (Z_1, Z_2 \ldots, Z_t),
    \qquad
    \mathcal{T} := \{ \lfloor m/4 \rfloor, \lfloor m/4 \rfloor + 1, \ldots, m-1 \},
\end{equation*}
and let $f_t := \operatorname{BB}_{S_{\leq t}}$. Now, with probability at least $1-\delta$, we have
\begin{equation*}
    \frac{1}{|\mathcal{T}|} \sum_{t \in \mathcal{T}} R_{\mathcal U,\mathcal{D}}(f_t) \leq C \, \frac{d+\log(1/\delta)}{m}.
\end{equation*}
We return their pointwise majority,
\begin{equation*}
    \widehat{h}(x) := \mathds{1} \left\{ \frac{1}{|\mathcal{T}|} \sum_{t \in \mathcal{T}}f_t(x) \geq \frac{1}{2} \right\}.
\end{equation*}
Indeed, if $\widehat{h}$ makes a robust error on an example, then some
perturbation is misclassified by at least half of the $f_t$. Consequently,
\begin{equation*}
    \ell_{\mathcal{U}}(\widehat{h};Z) \leq \frac{2}{|\mathcal{T}|} \sum_{t \in \mathcal{T}} \ell_{\mathcal{U}}(f_t;Z).
\end{equation*}
Taking expectations proves the claimed high-probability guarantee. The case $d=0$ is also immediate.

\paragraph{Sparsification and parallelization.} The exact aggregate as in Algorithm~\ref{alg:binomial-bagging} averages over all $2^m$ subsets and is therefore computationally inefficient. However, it can be approximated using only a small number of randomly sampled subsets. In particular, conditional on a fixed sample $S$, let $\mu_S$ denote the distribution on $\mathcal{H}$ obtained by retaining each example independently with probability $1/2$ and returning the corresponding RERM. Draw
\begin{equation*}
    h_1, h_2, \ldots, h_N \stackrel{\mathrm{i.i.d.}}{\sim} \mu_S,
\end{equation*}
and define
\begin{equation*}
    \widetilde{b}_S(x) := \frac{1}{N} \sum_{j=1}^N h_j(x).
\end{equation*}
Let $d^\star := \operatorname{VC} \bigl( \{h \mapsto h(x): x \in \mathcal{X} \} \bigr)$ denote the dual VC dimension of $\mathcal{H}$. Now, standard VC uniform convergence applied to this class implies that for every $\tau, \eta \in (0, 1)$,
$$
    N \geq C \, \frac{d^\star+\log(2/\eta)}{\tau^2}
$$
suffices to guarantee, with probability at least $1-\eta$ over the sampled subsets,
$$
    \sup_{x \in \mathcal{X}} \left| \widetilde b_S(x) - b_S(x) \right| \leq \tau.
$$
Thus, for constant $\tau$, only
$$
    N \in \mathcal{O} \bigl( d^\star+\log(1/\eta) \bigr)
$$
RERM calls suffice to approximate the exact aggregate uniformly. In particular, the resulting majority vote agrees with the exact binomial-bagging vote at every point $x \in \mathcal{X}$ for which
$$
    \left| b_S(x)-\frac{1}{2} \right| > \tau.
$$

Now, fix $\tau=1/4$ and call an example $(x,y)$ bad for a score $b$ if $|b(z)-y| \geq 1/4$ for some $z \in \mathcal{U}(x)$. For each bad leave-one-out example, at least one quarter of the subset classifiers omitting that example misclassify a fixed witness, whereas every subset containing it classifies that witness correctly. The preceding counting argument, with different numerical constants, therefore bounds the number of bad leave-one-out examples by $\mathcal{O}(d)$. Applying the prefix-averaging argument to this binary bad-event loss, with failure probability $\delta/2$, bounds its average population value by $\mathcal{O}((d+\log(1/\delta))/m)$. Next, approximate each prefix score $b_{S_{\leq t}}$, for $t \in \mathcal{T}$, uniformly to accuracy $1/4$, assigning failure probability $\delta/(2|\mathcal{T}|)$ to each approximation. This requires $N \in \mathcal{O}(d^\star+\log(2m/\delta))$ RERM calls per prefix, and a union bound guarantees that all approximations hold simultaneously with probability at least $1-\delta/2$, conditional on $S$. On this event, every robust mistake of an approximate prefix vote implies that the example is bad for the corresponding exact score, since $|\widetilde b(z)-y|\geq1/2$ implies $|b(z)-y|\geq1/4$. Consequently, the same majority-vote inequality shows that the pointwise majority $\widetilde{h}$ of the approximate prefix votes satisfies $R_{\mathcal{U},\mathcal{D}}(\widetilde{h}) \leq C \, (d+\log(1/\delta))/m$ with probability at least $1-\delta$ over both the training sample and the sampled subsets.

Moreover, a practical advantage of binomial bagging is that all constituent RERM calls are independent once the original sample is fixed, and can therefore be carried out in parallel. This applies both to the exact version, which considers all subsets, and to the sparse version above, which uses only randomly sampled subsets. In contrast, the boosting-based improper learner of \cite{pmlr-v99-montasser19a} is sequential, since the distribution used at each round depends on the hypotheses produced in the preceding rounds. More generally, \cite{karbasi2024impossibility} show that, for
boosting algorithms accessing a weak learner as a black box,
substantially reducing the number of sequential rounds can
require exponentially more calls to that learner.

\paragraph{Optimal agnostic rates via binomial bagging with a learned threshold.} In the agnostic setting, a fixed majority threshold is too crude, so we retain the binomial-bagging scores and learn where to threshold them. The analysis has three steps. First, we prove a deterministic leave-one-out bound: the total truncated loss is at most the best achievable number of robust errors on the sample, plus $\mathcal O(d/\alpha)$. Comparing each subset RERM with a RERM on the full sample controls the errors on retained examples, while a standard VC Rademacher bound controls the difference between retained and omitted errors at fixed witnesses. Subtracting a level $\alpha>0$ from each robust score loss and taking the positive part turns the resulting $\sqrt{dm}$ remainder into $\mathcal O(d/\alpha)$, independent of the sample size. Second, we average the scores over a range of sample prefixes. Reading the resulting sequence of losses forwards exhibits each as an unbiased estimate of its population risk, while reading it backwards exposes averages of leave-one-out terms; elementary martingale concentration then yields the desired high-probability control. Third, we learn a single threshold on a held-out block of the sample using \cite[Lemma~3.3]{engelund2026optimal}. Optimizing the proof parameter $\alpha$ only at the end gives the first-order guarantee in \autoref{theorem:introduction_1}; the learner itself never uses $\alpha$. The formal proof appears in \autoref{sec:agnostic-proof}.

\subsection{Further Implications} \label{sec:implications}

In this subsection, we briefly highlight that binomial bagging and its analysis have consequences beyond our main result.

\paragraph{Short proofs of optimal classical PAC bounds.} The special case of the identity perturbation map, $\mathcal{U}(x) = \{x\}$, makes adversarially robust PAC learning equivalent to classical PAC learning. In this case, RERM reduces to ERM, and binomial bagging becomes a simple aggregation of ERMs trained on independently retained sub-samples. Under the realizability assumption, our argument therefore gives a particularly short and elementary proof of the optimal
\begin{equation*}
\mathcal{O} \left( \frac{d + \log(1/\delta)}{\epsilon} \right)
\end{equation*}
sample complexity, first obtained by \cite{hanneke2016optimal}, and subsequently achieved by \cite{larsen2023bagging, aden2023optimal, rawal2026majority, engelund2026optimal}.

\newpage

In the agnostic setting, our argument provides a short and simple route to the optimal first-order bound, by replacing key components of the very recent proof of \cite{engelund2026optimal} while essentially retaining the final part of their argument. We refer interested readers to the works of \cite{engelund2026optimal, hanneke2024revisiting, asilis2026agnostic, devroye1996probabilistic} for a comprehensive discussion on the importance of this problem.

\paragraph{Proper PAC learning and the projection number.}
Let $\mathcal{H}$ be a concept class with $\operatorname{VC}$ dimension $d < \infty$ and projection number $k_p < \infty$. The projection number is the smallest integer $k \geq 2$ such that, for every non-empty finite multiset of hypotheses from $\mathcal{H}$, there exists a hypothesis in $\mathcal{H}$ agreeing with their pointwise majority wherever fewer than a $1/k$ fraction disagree with that majority. The work of \cite{bousquet2020proper} showed that $\mathcal{H}$ admits a proper realizable PAC learner with sample complexity
\begin{equation*}
    \mathcal{O} \left( \frac{k_p^2 \bigl(d\log(k_p)+\log(1/\delta)\bigr)}{\epsilon} \right).
\end{equation*}
We show that binomial bagging combined with this projection property improves the bound to
\begin{equation*}
    \mathcal{O} \left( \frac{k_p d\log(k_p)+\log(1/\delta)}{\epsilon} \right),
\end{equation*}
and that the improved guarantee continues to hold in the adversarially robust setting.

To see the main idea, fix a deterministic projection rule depending only on the input multiset, and apply it to $(h_A)_{A\subseteq[m]}$, retaining all multiplicities. This produces a proper learner. Now, fix a robustly realizable sample $S$ and suppose that, after removing example $(x_i,y_i)$, the resulting hypothesis makes a robust error on this example, witnessed by some $z_i \in \mathcal{U}(x_i)$. Then, at least a $1/k_p$ fraction of the RERMs trained on subsets whose index sets omit $i$ must misclassify $z_i$. Otherwise, their majority prediction at $z_i$ would be $y_i$, with fewer than a $1/k_p$ fraction disagreeing with it, and the projection property would force the projected hypothesis to predict $y_i$ as well.

The counting argument used in the realizable proof extends directly to this smaller fraction. More precisely, if $q$ examples each have a witness that is misclassified by at least a $\tau$ fraction of the RERMs trained without that example, the same two counting steps give
\begin{equation*}
    2^{\tau q / 4} \leq \frac{4}{\tau} \sum_{j=0}^{\min\{d,q\}} \binom{q}{j},
\end{equation*}
and hence
\begin{equation*}
    q \in \mathcal{O} \left( \frac{d}{\tau} \, \log\frac{2}{\tau} \right).
\end{equation*}
Taking $\tau=1/k_p$ therefore gives a deterministic leave-one-out bound of
\begin{equation*}
    \mathcal{O} \bigl( k_p d\log(k_p) \bigr).
\end{equation*}

Finally, this leave-one-out bound can be converted to high probability while preserving properness. We apply the same high-probability conversion to obtain a collection of proper hypotheses whose average robust error is
\begin{equation*}
    \mathcal{O} \left( \frac{k_p d\log(k_p)+\log(1/\delta)}{m} \right).
\end{equation*}
Hence, at least half of these hypotheses have robust error within a constant factor of this quantity. Sampling $\mathcal{O}(\log(1/\delta))$ of them and using an independent part of the sample to select the one with the smallest empirical robust error yields, with probability at least $1-\delta$, a hypothesis in $\mathcal{H}$ with robust error
\begin{equation*}
    \mathcal{O} \left( \frac{k_p d\log(k_p)+\log(1/\delta)}{m} \right).
\end{equation*}
This proves the claimed sample-complexity upper bound.


\subsection{Additional Related Work} \label{sec:related-work}

\paragraph{Bagging.} Bagging, introduced in a pioneering work of \cite{breiman1996bagging}, trains a base learner on bootstrap resamples of the training data and aggregates the resulting predictors by averaging or majority voting. Several variants modify the resampling mechanism, including \emph{subagging}, which uses sub-samples drawn without replacement \cite{buhlmann2002analyzing}; and \emph{online bagging}, which uses Poisson multiplicities to approximate bootstrap resampling \cite{oza2001online}.

Recently, in learning theory, the seminal work of \cite{larsen2023bagging} proved that classical bagging of empirical risk minimizers achieves the optimal realizable PAC learning sample complexity using only logarithmically many bootstrap samples. Most closely related to our resampling scheme, \cite{wu2025ensemble} studies bagged minimum-norm least-squares estimators under a multiplier-bootstrap framework that includes a \emph{Bernoulli bootstrap}, in which each training example is independently retained with probability $p$. They show that, in the proportional asymptotic regime, bagging suppresses the prediction-risk blow-up exhibited by individual sub-sampled estimators near the interpolation threshold, while also inducing implicit regularization. Our \emph{binomial bagging} employs the same resampling mechanism with $p=1/2$, so that the subsample size is binomially distributed; the exhaustive version, which aggregates over all subsets, is the exact majority vote induced by this resampling distribution.

\paragraph{Tolerant adversarial robustness.} Tolerant adversarially robust PAC learning, introduced by \cite{ashtiani2023adversarially}, relaxes the benchmark against which the learner is compared in our framework. In particular, given nested perturbation maps $\mathcal{U}, \mathcal{V}$, where $\mathcal{U}(x) \subseteq \mathcal{V}(x)$ for every $x \in \mathcal{X}$, the learner is evaluated under $\mathcal{U}$ while competing with the best hypothesis in $\mathcal{H}$ under $\mathcal{V}$; that is, the goal is to guarantee
\begin{equation*}
    R_{\mathcal{U},\mathcal{D}}(\widehat{h}) \leq R_{\mathcal{V},\mathcal{D}}^\star(\mathcal{H}) + \epsilon.
\end{equation*}
For metric perturbations, given $r > 0$, the canonical choice is
$\mathcal{U}(x) = B_r(x)$ and $\mathcal{V}(x) = B_{(1+\gamma) \, r}(x)$, where $\gamma > 0$ is the tolerance parameter and $B_r(x)$ denotes the ball of radius $r$ centered at $x$.

The original work of \cite{ashtiani2023adversarially} established guarantees with essentially linear dependence on the $\operatorname{VC}$ dimension and, for its compression-based agnostic PAC learner, on the doubling dimension. Under additional geometric regularity conditions, a tolerant variant of the robust empirical risk minimizer has been shown to be effective for bounded perturbation regions in $\mathbb{R}^d$ \cite{bhattacharjee2023robust}. Related work studied this topic alongside other relaxations interpolating between average-case and worst-case adversarial robustness \cite{raman2023proper}. More recently, simple ``almost proper'' learners have achieved essentially linear dependence on the $\operatorname{VC}$ dimension without additional structural assumptions on $\mathcal{H}$, while fully proper learning has been shown to be impossible in general even for classes of $\operatorname{VC}$ dimension one \cite{ashtiani2025simplifying}. These positive results exploit the nonzero gap between the perturbation sets used for evaluation and comparison. In contrast, our results address the zero-tolerance case $\mathcal{U} = \mathcal{V}$.

\subsection{Concurrent Work} \label{sec:concurrent-work}

During the preparation of this manuscript, we became aware of the concurrent work of \cite{montasser2026baggingrobustlylearnsvc}, which studies the same problem. All results presented in this work were obtained independently and established before the concurrent preprint became publicly available.

In the realizable setting, Montasser obtains the same optimal sample complexity using bootstrap bagging of RERMs. The two algorithms and their analyses are distinct: Montasser aggregates RERMs trained on independent bootstrap samples drawn from random prefixes of the data and analyzes the resulting ensemble using a second-moment argument, whereas our binomial-bagging approach aggregates RERMs over sub-samples obtained by independently retaining each example with probability $1/2$, and establishes a deterministic $\mathcal{O}(d)$ leave-one-out bound through an elementary counting argument based on the Sauer--Shelah--Perles lemma.

In the agnostic setting, Montasser applies the agnostic-to-realizable reduction of \cite{pmlr-v99-montasser19a}, building on \cite{moran2016sample, david2016supervised}, and obtains
\[
    R_{\mathcal{U},\mathcal{D}}(\widehat{h}) \leq R_{\mathcal{U},\mathcal{D}}^{\star}(\mathcal{H}) + C \left( \sqrt{\frac{d \log^2 m + \log(1/\delta)}{m}} \right).
\]
Indeed, the same reduction can be combined with our realizable learner (as well as the reduction from \cite{pmlr-v178-hopkins22a}; also see \cite{hanneke2025representation}), but it loses logarithmic factors and does not yield the optimal first-order refinement.

The concurrent work also provides a complementary oracle-complexity result. Its realizable learner uses $\mathcal{O}(d^\star + \log(1/\delta))$ calls to a RERM oracle, and Montasser proves that, in the corresponding black-box oracle model, there exist classes for which $\Omega(d^\star)$ calls are necessary at constant confidence, regardless of the number of available samples. By contrast, our main emphasis is on the binomial-bagging principle and its elementary analysis, which may be of independent interest and already yields several immediate consequences that are discussed in \autoref{sec:implications}.

\subsection{Conclusion, Discussion, and Future Directions} \label{sec:conclusion}

In this paper, we showed that adversarial robustness incurs no additional distribution-free statistical cost: uniformly over all perturbation maps $\mathcal{U}$, the optimal sample complexities of adversarially robust PAC learning in both the realizable and agnostic settings match those of ordinary PAC learning, up to universal constant factors. In particular, for every concept class $\mathcal{H}$ of $\operatorname{VC}$ dimension $d$, we prove upper bounds of $\mathcal{O} \big( d/\epsilon + \log(1/\delta)/\epsilon \big)$ in the realizable setting and $\mathcal{O} \big( d/\epsilon^2 + \log(1/\delta)/\epsilon^2 \big)$ in the agnostic setting, which match the corresponding lower bounds for classical PAC learning. Notably, our agnostic result admits a first-order refinement that is optimal up to universal constant factors, matching the corresponding lower bound for ordinary PAC learning as well.

Technically, our results follow from short and elementary proofs based on a new algorithmic principle that we call binomial bagging. In this algorithm, the learner aggregates suitable empirical risk minimizers trained on sub-samples obtained by retaining each training example independently with probability $1/2$. We believe that both this principle and its analysis may be of independent interest.

Finally, we outline several avenues for future research that naturally arise from our study and may be worth exploring further.

\begin{itemize}[label=*]
    \item Our results characterize the optimal sample complexity uniformly over perturbation maps $\mathcal{U}$, but for a fixed pair $(\mathcal{H}, \mathcal{U})$ the statistical complexity can be substantially smaller \cite{pmlr-v99-montasser19a}. A natural direction is therefore to resolve Conjecture~1 of \cite{montasser2022adversarially}, which proposes an optimal characterization of the $\mathcal{U}$-dependent sample complexity in terms of a new combinatorial complexity dimension.

    \item More broadly, a natural direction is to explore further applications of binomial bagging and its analysis across learning theory, beyond the consequences discussed in \autoref{sec:implications}.

    \item The binomial-bagging principle suggests a potential bridge between learning theory and practical machine learning research. In particular, in modern AI systems, RERM can be approximated by adversarial training. This raises the intriguing question of whether combining such approximate RERMs through methods inspired by binomial bagging can improve robust generalization or mitigate robust overfitting.
\end{itemize}

\subsection{Organization} \label{sec:organization}

The remainder of the paper is organized as follows. In \autoref{sec:preliminaries}, we formally set the basic notation, provide definitions, and present some preliminaries. Subsequently, in \autoref{sec:agnostic-proof}, we give a proof of Theorem~\ref{theorem:introduction_1}.

%% file: Main/Preliminaries.tex
\section{Preliminaries} \label{sec:preliminaries}

In this section, we first set our basic notation in \autoref{sec:notation}. We then define the adversarially robust PAC learning framework in \autoref{sec:pac-framework}. Finally, we give the definition of the combinatorial complexity parameter considered in this work in \autoref{sec:combinatorial-complexity-parameter}.

\subsection{Notation} \label{sec:notation}

In this subsection, we present the basic notation used throughout the paper; all of it is standard in the literature and is included for completeness. Let $\mathbb{N}$ and $\mathbb{R}$ stand for the set of natural numbers and real numbers, respectively. In addition, for a given $n \in \mathbb{N}$, we use $[n]$ to denote $\big \{ 1, 2, \ldots, n \big \}$, and set $[0]:=\varnothing$. Next, given $n \in \mathbb{N}$, for any sequence of size $n$ or $n$-tuple $v$, and any $i \in [n]$, let us use $v_i$ to denote the $i$-th element in $v$. Afterward, we denote by $A \times B$ the Cartesian product of two arbitrary sets $A$ and $B$. In addition, for any set $A$ and any $n \in \mathbb{N}$, we let $A^n$ indicate $n$ times the Cartesian product of $A$ with itself. Note that for any set $A$, we define $A^0 := \{\emptyset\}$. Also, given a set $A$, we denote by $A^{*}$ the set of all finite sequences of members of $A$; more formally, $A^{*} := \bigcup^{\infty}_{T = 0} A^T$. Then, for arbitrary sets $X$ and $Y$, we use $Y^X$ to denote the space of all functions from $X$ to $Y$. Going further, let $\mathrm{e}$, $\log{(\cdot)}$, and $\ln{(\cdot)}$ stand for Euler's number, the Logarithm function in base $2$, and the natural logarithm, respectively. Finally, we use $\mathcal{O}(\cdot)$, $\Omega(\cdot)$, and $\Theta(\cdot)$ to denote the standard asymptotic notation in theoretical computer science.

\subsection{Adversarially Robust PAC Learning Framework} \label{sec:pac-framework}

We begin by giving the problem setup. Fix a non-empty measurable space $(\mathcal{X}, \Sigma_{\mathcal{X}})$, referred to as the instance space. For example, the instance space $\mathcal{X}$ can be a Euclidean space. Also, fix $\mathcal{Y} = \{0, 1\}$ as the label space. A concept is a measurable function from $\mathcal{X}$ to $\mathcal{Y}$. With this in mind, fix a non-empty collection of such functions $\mathcal{H}$ as the concept class. Furthermore, fix a perturbation map $\mathcal{U} : \mathcal{X} \to 2^{\mathcal{X}}$, where $\mathcal{U}(x)$ is the set of allowable perturbations of an instance $x$. For example, the perturbation map $\mathcal{U}$ can be defined as $\mathcal{U}(x) = \{ z \in \mathcal{X} \; | \; \lVert z-x \rVert_{\infty} \leq \rho \}$, where $\rho > 0$ is a small radius. We assume throughout that, for every $A \in \Sigma_{\mathcal{X}}$, the set $\{ x \in \mathcal{X} \mid A \, \cap \, \mathcal{U}(x) \neq \varnothing \}$ belongs to $\Sigma_{\mathcal{X}}$. For simplicity, we specify instances of the adversarially robust PAC learning framework solely based on the concept class $\mathcal{H}$, which is defined relative to the measurable structure $(\mathcal{X}, \Sigma_{\mathcal{X}})$, and the perturbation map $\mathcal{U}$.

Next, we present the definition of a data distribution. A data distribution $\mathcal{D}$ is a probability measure over $\mathcal{X} \times \mathcal{Y}$. For example, it may be a probability measure over the space of words in a dictionary and their corresponding binary labels.

Subsequently, we present the definition of a learning algorithm. A PAC learning algorithm $\mathbf{A}$ is a mapping from $(\mathcal{X} \times \mathcal{Y})^{*} \times \mathcal{X}$ to $\mathcal{Y}$ whose restriction to $(\mathcal{X} \times \mathcal{Y})^n \times \mathcal{X}$ is universally measurable for every $n \in \mathbb{N} \cup \{0\}$. In words, a PAC learning algorithm is a function that maps a finite sequence of examples, and an instance, to a label. 

We also require the robust loss of its output to be jointly universally measurable in the training sample and test example.

Given a PAC learning algorithm $\mathbf{A}$, for every dataset $\mathcal{S} \in (\mathcal{X} \times \mathcal{Y})^{n}$, where $n \in \mathbb{N} \cup \{0\}$, define $\mathbf{A}_{\mathcal{S}} : \mathcal{X} \to \mathcal{Y}$ by
\begin{equation*}
    \mathbf{A}_{\mathcal{S}}(x) := \mathbf{A} \big( (\mathcal{S}, x) \big) \qquad \text{for all} \; x \in \mathcal{X}.
\end{equation*}

Afterward, we introduce three definitions: the first concerns the true robust error of a concept, the second concerns the empirical robust error of a concept, and the last concerns the true robust error of a PAC learning algorithm.

Before starting, for a concept $f : \mathcal{X} \to \mathcal{Y}$ and an example $Z = (x,y)$, we write $\ell_{\mathcal U}(f;Z)$ for the robust error indicator of that example, as it appears in the sums below. Thus, repetitions in a sample are counted with multiplicity.

First, we define the \emph{true robust error} of a concept $h$ with respect to a perturbation map $\mathcal{U}$ and a data distribution $\mathcal{D}$, denoted by $R_{\mathcal{U},\mathcal{D}}(h)$, as follows:
\begin{equation*} \label{def:true-robust-error-concept}
    R_{\mathcal{U},\mathcal{D}}(h) := \mathbb{P}_{(X,Y) \sim \mathcal{D}} \Bigl[ \, \exists \, z \in \mathcal{U}(X) \; \text{such that} \; h(z) \neq Y \, \Bigr],
\end{equation*}

Second, we define the \emph{empirical robust error} of a concept $h$ with respect to a perturbation map $\mathcal{U}$ and a dataset $\mathcal{S} = \big( (x_1, y_1), (x_2, y_2), \ldots, (x_n, y_n) \big) \in (\mathcal{X} \times \mathcal{Y})^n$ for some $n \in \mathbb{N}$, denoted by $R_{\mathcal{U},\mathcal{S}}(h)$, as follows:
\begin{equation*} \label{def:emp-robust-error-concept}
    R_{\mathcal{U},\mathcal{S}}(h) := \frac{1}{n} \sum_{i=1}^n \mathds{1} \Bigl\{ \, \exists \, z \in \mathcal{U}(x_i) \; \text{such that} \; h(z) \neq y_i \, \Bigr\}.
\end{equation*}
Note that for the empty sample, we set $R_{\mathcal U,\varnothing}(h):=0$.

Finally, let $\mathbf{A}$ be a PAC learning algorithm. Given $\mathcal{S}$ as above, we define the \emph{true robust error} of $\mathbf{A}_{\mathcal{S}}$ with respect to a perturbation map $\mathcal{U}$ and a data distribution $\mathcal{D}$, denoted by $R_{\mathcal{U},\mathcal{D}}(\mathbf{A}_{\mathcal{S}})$, as follows: 
\begin{equation*} \label{def:true-robust-error-algorithm}
    R_{\mathcal{U},\mathcal{D}}(\mathbf{A}_{\mathcal{S}}) := \mathbb{P}_{(X,Y) \sim \mathcal{D}} \Bigl[ \, \exists \, z \in \mathcal{U}(X) \; \text{such that} \; \mathbf{A}_{\mathcal{S}}(z) \neq Y \, \Bigr].
\end{equation*}

Building on the previous definitions, we now define adversarially robust agnostic PAC learnability. This definition is based on the same principles as those introduced by \cite{valiant1984theory, vapnik1974theory}.

\begin{definition}[Adversarially robust agnostic PAC learnability] \label{def:adversarially-robust-agnostic-pac-learnability}
We say that a concept class $\mathcal{H}$ together with a perturbation map $\mathcal{U}$ is adversarially robust agnostic PAC learnable, if there exists a PAC learning algorithm $\mathbf{A}$ and a sample complexity function $m^{\mathbf{A}}_{\mathcal{H},\mathcal{U}}: (0, 1) \times (0, 1) \to \mathbb{N}$ such that, for every data distribution $\mathcal{D}$, every $\epsilon \in (0, 1)$, and every $\delta \in (0, 1)$, for $n \geq m^{\mathbf{A}}_{\mathcal{H},\mathcal{U}}(\epsilon, \delta)$, we have:
\begin{equation*}
    \mathbb{P}_{\mathcal{S} \sim \mathcal{D}^n} \, \big[ \, R_{\mathcal{U},\mathcal{D}}(\mathbf{A}_{\mathcal{S}}) \leq \, \inf_{h \in \mathcal{H}} \, R_{\mathcal{U},\mathcal{D}}(h) + \epsilon \, \big] \geq 1 - \delta.
\end{equation*}

Now, let $\mathbf{A}$ be a PAC learning algorithm. We define the sample complexity of $\mathbf{A}$ as the function $\mathbf{m}_{\mathcal{H},\mathcal{U}}^{\mathbf{A}} : (0, 1) \times (0, 1) \to \mathbb{N}\cup\{+\infty\}$ such that, for every $\epsilon \in (0,1)$ and $\delta \in (0, 1)$, the value $\mathbf{m}_{\mathcal{H},\mathcal{U}}^{\mathbf{A}}(\epsilon, \delta)$ is the pointwise infimum over all functions satisfying the above guarantee for $\mathbf{A}$. We use the convention $\inf\varnothing=+\infty$.

Finally, the optimal sample complexity of adversarially robust agnostic PAC learning $\mathcal{H}$ together with $\mathcal{U}$, denoted by $\mathbf{m}_{\mathcal{H},\mathcal{U}}: (0, 1) \times (0, 1) \to \mathbb{N}\cup\{+\infty\}$, is then defined as follows:
\begin{equation*}
    \mathbf{m}_{\mathcal{H},\mathcal{U}}(\epsilon, \delta) := \inf_{\mathbf{A}} \, \mathbf{m}_{\mathcal{H},\mathcal{U}}^{\mathbf{A}}(\epsilon, \delta) \qquad \text{for all} \; \epsilon, \delta \in (0, 1).
\end{equation*}
that is, given $\epsilon, \delta \in (0, 1)$, the minimal number of samples sufficient for the existence of some PAC learning algorithm that satisfies the above guarantee.
\end{definition}

At the end of this subsection, we define a robust empirical risk minimizer, which will play a central role in our algorithms. 

\begin{definition}[Robust empirical risk minimizer (RERM)]
Fix a concept class $\mathcal{H} \subseteq \{0,1\}^{\mathcal{X}}$ and a perturbation map $\mathcal{U}$. A \emph{robust empirical risk minimizer (RERM)} is a mapping
\begin{equation*}
    \operatorname{RERM}_{\mathcal{H},\mathcal{U}} : (\mathcal{X} \times \mathcal{Y})^{*} \to \mathcal{H}
\end{equation*}
such that, for every dataset $\mathcal{S} \in (\mathcal{X} \times \mathcal{Y})^{*}$,
\begin{equation*}
    \operatorname{RERM}_{\mathcal{H},\mathcal{U}}(\mathcal{S}) \in \argmin_{h\in\mathcal{H}} R_{\mathcal{U},\mathcal{S}}(h).
\end{equation*}
When there are multiple minimizers, the RERM may select any one of them.
\end{definition}

Throughout the paper, we fix a deterministic RERM whose output
depends only on the input multiset, and write $\mathsf{RERM}$ for
this rule. Thus it is permutation-invariant, including its fixed
choice on the empty sample. On a standard Borel instance space,
one may also obtain permutation invariance by sorting the
instance-label pairs using a measurable total order before applying
a fixed RERM \cite[Appendix~D]{hanneke2021universal}.

\paragraph{Measurability.} We assume that the chosen RERM has jointly measurable evaluation in the sample and query. For every score $q_S$ constructed in Section~\ref{sec:agnostic-proof}, including RERM outputs, binomial-bagging scores, and finite averages, we also assume that
\begin{equation*}
    (S,x) \longmapsto \sup_{z \in \mathcal{U}(x)} q_S(z)
    \quad \text{and} \quad
    (S,x) \longmapsto \inf_{z \in \mathcal{U}(x)} q_S(z)
\end{equation*}
are jointly measurable, taking values $0$ and $1$, respectively, when $\mathcal U(x)=\varnothing$. These are standing regularity assumptions on the class, perturbation map, and chosen RERM; they do not follow from finite VC dimension or permutation invariance. They ensure that the random losses and risks in the proof are measurable.

\subsection{Combinatorial Complexity Parameter} \label{sec:combinatorial-complexity-parameter}

In this subsection, we introduce the combinatorial complexity parameter that appears in our results, namely the $\operatorname{VC}$ dimension.

\begin{definition}[Shattered set] \label{def:shattered-set}
Let $\mathcal{H} \subseteq \{0,1\}^\mathcal{X}$ be a concept class, and let $S \subseteq \mathcal{X}$ be a finite set of instances. We say that $S$ is \emph{shattered} by $\mathcal{H}$ if for every $f : S \to \{0,1\}$, there exists a concept $h \in \mathcal{H}$ such that $h(x) = f(x)$ for all $x \in S$.
\end{definition}

\begin{definition}[$\operatorname{VC}$ dimension] \label{def:vc-dim}
Let $\mathcal{H} \subseteq \{0,1\}^\mathcal{X}$ be a concept class. The \emph{$\operatorname{VC}$ dimension} of $\mathcal{H}$, denoted by $\operatorname{VC}(\mathcal{H}) \in \mathbb{N} \cup \{0, +\infty\}$, is defined by the $\sup_{d \in \mathbb{N} \cup \{0\}}$ such that, there exists a finite set of instances $S \subseteq \mathcal{X}$ of size $d$ that is shattered by $\mathcal{H}$. Also, if $\mathcal{H} = \emptyset$, we have: $\operatorname{VC}(\mathcal{H}) = 0$.
\end{definition}

%% file: Main/Proofs.tex
\section{Proof of the Main Result}\label{sec:agnostic-proof}

Throughout this section, we fix a distribution $\mathcal D$ and write
\[
    R(h):=R_{\mathcal U,\mathcal D}(h),
    \qquad
    R^\star:=\inf_{h\in\mathcal H}R(h),
    \qquad
    r(T):=\min_{h\in\mathcal H}\sum_{Z\in T}\ell_{\mathcal U}(h;Z),
\]
for every $h\in\mathcal H$ and every finite sample
$T\in(\mathcal X\times\mathcal Y)^{*}$, where the occurrences in $T$ are
counted with multiplicity. Thus $r(T)$ is the number of robust mistakes
that the best hypothesis in the class makes on $T$. We also write $b_T$
for the binomial-bagging score of $T$ from
Algorithm~\ref{alg:binomial-bagging}, and $T^{-i}$ for the sample
obtained from $T$ by deleting its $i$-th occurrence.

\paragraph{Plan of the proof.} In the realizable case
(Section~\ref{sec:technical-overview}) the majority vote of the
binomial-bagging score suffices, because it makes only $\mathcal O(d)$
leave-one-out mistakes. Agnostically a majority vote is too blunt:
rounding the score at $1/2$ discards exactly what we need, namely how
close to the boundary each prediction was. We therefore keep the score
and analyze it directly, in three steps.

First (Section~\ref{sec:deterministic-agnostic-bound}), a deterministic
leave-one-out bound: deleting an example and scoring it with the score
built from the rest costs, in total, at most $r(S)+\mathcal O(d/\alpha)$,
provided each term is first \emph{truncated} at a level $\alpha>0$.
Second (Section~\ref{sec:agnostic-conversion-main}), the learner averages
the scores of a range of prefixes, which lets us read one sequence of
losses in two directions of time --- forwards each is an unbiased
estimate of its population value, backwards each is an average of
leave-one-out terms --- and this turns the first step into a
high-probability bound. Third, a threshold chosen on a held-out block
turns the score back into a classifier at the cost of a factor
$1+\mathcal O(\alpha)$. Optimizing $\alpha$ at the very end gives the
first-order rate; the learner never sees it.

\subsection{The Learner} \label{sec:agnostic-learner}

Two degenerate cases are disposed of first. If $\mathcal H$ consists of a
single hypothesis, the learner returns it and attains $R^\star$ exactly; this
covers $d=0$. If $m<6$, the learner returns $\mathsf{RERM}(\varnothing)$
and ignores the sample, which is admissible because the robust loss is
bounded by $1$ while the error term of
Theorem~\ref{theorem:introduction_1} is at least $C(d+\log(1/\delta))/m
\geq Cd/6$, so the claim holds once $C$ is large enough. For the rest of
this section we assume $d\geq1$ and $m\geq6$, and the learner is the one
in Algorithm~\ref{alg:learned-threshold}.

\begin{figure}[!t]
\captionsetup{skip=5pt}
\centering
\begin{tcolorbox}
\begin{center}
\textbf{Binomial Bagging with a Learned Threshold}
\end{center}
\noindent \textbf{Input:} A sample
\[
    S=(Z_1,Z_2,\ldots,Z_m)\in(\mathcal X\times\mathcal Y)^m,
    \qquad m\geq6.
\]
\begin{enumerate}
    \item Set $k:=\lfloor m/3\rfloor$. For each $t=k,k+1,\ldots,2k-1$,
    compute the binomial-bagging score of the length-$t$ prefix, namely,
    \[
        S_{\leq t}:=(Z_1,Z_2,\ldots,Z_t),
        \qquad b_t:=b_{S_{\leq t}}.
    \]
    \item Average these $k$ scores, namely,
    \[
        g_S:=\frac1k\sum_{t=k}^{2k-1}b_t.
    \]
    \item On the held-out block
    $V:=(Z_{2k+1},Z_{2k+2},\ldots,Z_{3k})$, take the largest threshold
    minimizing the empirical robust error, namely,
    \[
        \widehat u:=\max\argmin_{u\in[0,2]}
        R_{\mathcal U,V}
        \bigl(\mathds{1}\{g_S\geq u\}\bigr).
    \]

    \item Return the classifier
    $\widehat h:\mathcal X\rightarrow\mathcal Y$ such that for every
    $x\in\mathcal X$, we have
    \[
        \widehat h(x):=\mathds{1}\{g_S(x)\geq\widehat u\}.
    \]
\end{enumerate}
\end{tcolorbox}
\caption{Binomial bagging with a learned threshold. The score is built by
bagging over prefixes of the sample, and its threshold is learned on a
held-out block.}
\label{alg:learned-threshold}
\end{figure}

The first $2k-1$ examples build the prefix-averaged score, the next
example $Z_{2k}$ is not used by the learner, the following $k$ examples
form the validation block $V$, and the at most two examples beyond
$Z_{3k}$ are discarded. Averaging the prefix scores, rather than using
the single score $b_{S_{\leq 2k-1}}$, is what makes the concentration
argument of Section~\ref{sec:agnostic-conversion-main} possible; a single
score does not provide the martingale sequence used in our
high-probability conversion. The threshold ranges over $[0,2]$ rather
than $[0,1]$ for two reasons: it makes both constant classifiers
available ($u=0$ gives $\widehat h\equiv1$ and $u=2$ gives
$\widehat h\equiv0$), and it accommodates the sentinel encoding used in
the proof of Lemma~\ref{lem:truncated-threshold-selection}, which also
shows that the maximum defining $\widehat u$ is attained.

Although $Z_{2k}$ is invisible to the learner, it reappears in the
analysis as the fresh example against which the last prefix score
$b_{2k-1}$ is tested. The learner uses neither $d$ nor $r(S)$, so it
needs no knowledge of the complexity of the class or of the noise
level.

\subsection{The Truncated Loss} \label{sec:truncated-loss}

The object the learner builds, $g_S$, is a score rather than a
classifier, so the analysis needs a loss that accepts scores. For a
measurable $q:\mathcal X\to[0,1]$ set
\[
    \overline\ell_{\mathcal U}(q;(x,y))
       :=\sup_{z\in\mathcal U(x)}|q(z)-y|,
\]
with the value $0$ when $\mathcal U(x)=\varnothing$. The bar denotes the
worst case over the perturbation set, exactly as in $\ell_{\mathcal U}$,
and the two losses agree when $q$ takes values in $\{0,1\}$. For
$\alpha>0$ define the truncated loss and its population value
\[
    \phi_\alpha(q;Z):=
       \bigl(\overline\ell_{\mathcal U}(q;Z)-\alpha\bigr)_+,
    \qquad
    \Phi_\alpha(q):=\mathbb E_{Z\sim\mathcal D}\,\phi_\alpha(q;Z),
    \qquad
    (v)_+:=\max\{v,0\}.
\]
Truncation forgives the first $\alpha$ units of loss on every example,
and this is what buys Lemma~\ref{lem:truncated-loo}: the argument below
bounds the untruncated leave-one-out sum by an additive $\sqrt{d n}$ term
(inequality~\eqref{eq:loo-untruncated}), which grows with the sample,
whereas after truncation the same argument gives
$\mathcal O(d/\alpha)$, which does not. For later use we note that
$\phi_\alpha$ is convex in $q$, since
$\overline\ell_{\mathcal U}(\cdot;Z)$ is a supremum of convex functions of
$q$ and $v\mapsto(v-\alpha)_+$ is convex and nondecreasing. The parameter
$\alpha$ appears nowhere in the learner; it is fixed at the end of the
proof of Theorem~\ref{theorem:introduction_1}.

\subsection{A Leave-One-Out Bound} \label{sec:deterministic-agnostic-bound}

The following bound is deterministic: it holds for every sample, with no
distributional assumption whatsoever.

\begin{lemma}\label{lem:truncated-loo}
There is a universal constant $C$ such that, for every sample
$S=(Z_1,\ldots,Z_n)$ with $n\geq1$ and every $\alpha>0$,
\begin{equation}
    \sum_{i=1}^n\phi_\alpha(b_{S^{-i}};Z_i)
    \leq r(S)+\frac{Cd}{\alpha}.
    \label{eq:truncated-loo}
\end{equation}
\end{lemma}

\begin{proof}
Write $Z_i=(x_i,y_i)$, put
$p_i:=\overline\ell_{\mathcal U}(b_{S^{-i}};Z_i)$, and let
$I:=\{i:p_i>\alpha\}$ be the set of examples that survive truncation. If
$I=\varnothing$ the left-hand side of \eqref{eq:truncated-loo} vanishes
and there is nothing to prove, so assume $I\neq\varnothing$. Each score
$b_{S^{-i}}$ takes finitely many values, so for every $i\in I$ we may fix
a witness $z_i\in\mathcal U(x_i)$ attaining the supremum, that is,
$p_i=|b_{S^{-i}}(z_i)-y_i|$.

\circled{1} \textbf{Retained versus omitted examples.}
Draw $A\subseteq[n]$ uniformly at random and write
$h_A:=\mathsf{RERM}(S_A)$ and
$e_i:=\mathds{1}\{h_A(z_i)\neq y_i\}$. Let
$\varepsilon_i:=1-2\cdot\mathds{1}\{i\in A\}$, so that the $\varepsilon_i$
are independent uniform signs, equal to $+1$ exactly when $i$ is omitted.

Conditioned on $i\notin A$, the set $A$ is uniform over the subsets
omitting $i$, and this is precisely the distribution defining
$b_{S^{-i}}$. Hence
$b_{S^{-i}}(z_i)=\mathbb P[h_A(z_i)=1\mid i\notin A]$, and subtracting
$y_i$ turns this into a probability of error for either label:
\begin{equation}
    p_i=\mathbb P\bigl[h_A(z_i)\neq y_i\,\big|\,i\notin A\bigr],
    \qquad i\in I.
    \label{eq:p-as-probability}
\end{equation}

Omitted and retained examples behave very differently, and the proof
plays the two against each other. On the retained side,
$e_i\leq\ell_{\mathcal U}(h_A;Z_i)$ because $z_i\in\mathcal U(x_i)$, and
$h_A$ minimizes the robust empirical error on $S_A$ over the class, so
comparing it with $\mathsf{RERM}(S)$ gives
\begin{equation}
    \mathbb E\sum_{i\in I\cap A}e_i
    \leq\mathbb E\sum_{i\in A}\ell_{\mathcal U}(h_A;Z_i)
    \leq\mathbb E\sum_{i\in A}\ell_{\mathcal U}(\mathsf{RERM}(S);Z_i)
    =\frac{r(S)}2,
    \label{eq:retained-side}
\end{equation}
the last equality because each index lies in $A$ with probability $1/2$.
On the omitted side, \eqref{eq:p-as-probability} gives
\[
    \mathbb E\sum_{i\in I\setminus A}e_i
    =\sum_{i\in I}\mathbb P[i\notin A]\;
      \mathbb E\bigl[e_i\mid i\notin A\bigr]
    =\frac12\sum_{i\in I}p_i.
\]
Since $\sum_{i\in I}\varepsilon_ie_i$ is the omitted count minus the
retained count, taking expectations in
$\sum_{i\in I\setminus A}e_i=\sum_{i\in I\cap A}e_i
 +\sum_{i\in I}\varepsilon_ie_i$ and rearranging yields
\[
    \sum_{i\in I}p_i
    =2\,\mathbb E\sum_{i\in I\cap A}e_i
     +2\,\mathbb E\sum_{i\in I}\varepsilon_ie_i .
\]
The first term is at most $r(S)$ by \eqref{eq:retained-side}. In the
second, forget that the hypothesis is $h_A$ and pass to a supremum over
the class, which leaves a Rademacher average:
\begin{equation}
    \sum_{i\in I}p_i
    \leq r(S)+2\,\mathbb E\sup_{h\in\mathcal H}
        \sum_{i\in I}\varepsilon_i\mathds{1}\{h(z_i)\neq y_i\}
    \leq r(S)+C_{\mathrm R}\sqrt{d\,|I|}.
    \label{eq:loo-untruncated}
\end{equation}
The last step is the standard bound on the Rademacher average of a class
of VC dimension at most $d$
\cite[Sections~3.2 and~4.3]{devroye2001combinatorial}, with a universal
constant $C_{\mathrm R}$. The error indicators
$\mathds{1}\{h(z_i)\neq y_i\}$ do form such a class: flipping coordinates
according to the fixed labels $y_i$ and repeating witnesses $z_i$ leave
the VC dimension unchanged.

\circled{2} \textbf{Paying for the truncation.}
Every $i\notin I$ contributes $0$ to the left-hand side of
\eqref{eq:truncated-loo}, and every $i\in I$ contributes $p_i-\alpha$.
Subtracting $\alpha|I|$ from \eqref{eq:loo-untruncated},
\[
    \sum_{i=1}^n\phi_\alpha(b_{S^{-i}};Z_i)
    =\sum_{i\in I}(p_i-\alpha)
    \leq r(S)+C_{\mathrm R}\sqrt{d\,|I|}-\alpha|I|
    \leq r(S)+\frac{C_{\mathrm R}^2d}{4\alpha},
\]
where the last step maximizes $C_{\mathrm R}\sqrt d\,s-\alpha s^2$ over
$s=\sqrt{|I|}\geq0$; this is \eqref{eq:truncated-loo} with
$C:=C_{\mathrm R}^2/4$.
\end{proof}

\subsection{From Leave-One-Out to High Probability}
\label{sec:agnostic-conversion-main}

We first record the concentration inequality used in the argument. Both
halves are standard; we state them together because the proof applies
them in both directions of time.

\begin{lemma}[Multiplicative concentration]\label{lem:multiplicative-concentration}
Let $Y_1,\ldots,Y_N\in[0,1]$ be adapted to a filtration
$(\mathcal F_j)_{j=0}^N$, and write $\mu_j:=\mathbb E[Y_j\mid\mathcal F_{j-1}]$.
For fixed $0<\lambda\leq1$ and $L>0$, each of the inequalities
\[
    \sum_jY_j\leq(1+\lambda)\sum_j\mu_j+\frac L\lambda,
    \qquad
    \sum_j\mu_j\leq(1+\lambda)
                       \left(\sum_jY_j+\frac L\lambda\right)
\]
holds with probability at least $1-\mathrm{e}^{-L}$.
\end{lemma}

\begin{proof}
For $0<\lambda\leq1$ we have $\mathrm{e}^\lambda-1\leq\lambda+\lambda^2$
and $1-\mathrm{e}^{-\lambda}\geq\lambda-\lambda^2/2$. Since
$v\mapsto\mathrm{e}^{\pm\lambda v}$ is convex, its value on $[0,1]$ is at
most the corresponding chord, so taking conditional expectations gives
\[
    \mathbb E[\mathrm{e}^{\lambda Y_j}\mid\mathcal F_{j-1}]
      \leq1+(\mathrm{e}^\lambda-1)\mu_j
      \leq\mathrm{e}^{(\lambda+\lambda^2)\mu_j},
    \qquad
    \mathbb E[\mathrm{e}^{-\lambda Y_j}\mid\mathcal F_{j-1}]
      \leq\mathrm{e}^{-(\lambda-\lambda^2/2)\mu_j}.
\]
Multiplying these factors along the filtration and applying Markov's
inequality gives the first claim, and also
$(1-\lambda/2)\sum_j\mu_j\leq\sum_jY_j+L/\lambda$. The second claim
follows from $(1-\lambda/2)^{-1}\leq1+\lambda$.
\end{proof}

The next lemma is the third step: it converts a score into a classifier,
at the cost of a $1+\mathcal O(\alpha)$ factor over the truncated loss of
the score, plus an additive term of the right order.
The proof builds on \cite[Lemma~3.3]{engelund2026optimal}, after reducing robust threshold selection to an ordinary one-dimensional threshold problem.

\begin{lemma}[Threshold selection]\label{lem:truncated-threshold-selection}
Fix a measurable score $q:\mathcal X\to[0,1]$ with finite range, and let
$V\sim\mathcal D^n$ with $n\geq1$. Write $h_{q,u}:=\mathds{1}\{q\geq u\}$
and let $\widehat u\in[0,2]$ be the largest threshold minimizing the
robust empirical error on $V$. There is a universal constant $C$ such
that, for every $0<\alpha\leq1/8$ and every $\eta\in(0,1)$, with
probability at least $1-\eta$,
\begin{equation}
    R(h_{q,\widehat u})
    \leq(1+4\alpha)\Phi_\alpha(q)
           +C\frac{\ln(12/\eta)}{\alpha n}.
    \label{eq:truncated-threshold-selection}
\end{equation}
\end{lemma}

\begin{proof}
\circled{1} \textbf{Reduction to one-dimensional threshold ERM.}
Collapse each example to a single number, the score at which it starts
being misclassified. For $\mathcal U(x)\neq\varnothing$ set
\[
    a_q(x,y):=
    \begin{cases}
        \max_{z\in\mathcal U(x)}q(z),&y=0,\\[2pt]
        \min_{z\in\mathcal U(x)}q(z),&y=1,
    \end{cases}
\]
and for $\mathcal U(x)=\varnothing$ set $a_q(x,0):=-1$ and
$a_q(x,1):=3$. The extrema are attained because $q$ has finite range,
and the assumptions of Section~\ref{sec:preliminaries} make $a_q$
measurable. A direct check of the two labels gives, for every
$u\in[0,2]$,
\begin{equation}
    \ell_{\mathcal U}(h_{q,u};(x,y))
    =\mathds{1}\bigl\{\mathds{1}\{a_q(x,y)\geq u\}\neq y\bigr\}.
    \label{eq:threshold-reduction}
\end{equation}
Indeed, a label-$0$ example is robustly misclassified exactly when some
perturbation has score at least $u$, that is, when
$\max_zq(z)\geq u$; a label-$1$ example exactly when some perturbation
has score below $u$, that is, when $\min_zq(z)<u$. An example with
$\mathcal U(x)=\varnothing$ is never a robust mistake, and the sentinel
values $-1$ and $3$ reproduce this for every $u\in[0,2]$, since
$-1<u\leq3$. Any range of thresholds inside $(-1,3]$ would serve here,
and the range $[0,2]$ used by the learner is one such.

So selecting $\widehat u$ is ordinary threshold ERM on the pairs
$(a_q(Z),Y)$. As a function of $u$, the empirical error is piecewise
constant and left-continuous, with breakpoints only at the observed
values $a_q(Z_i)$, so its set of minimizers in $[0,2]$ is a finite union
of intervals closed on the right. The maximum $\widehat u$ is therefore
attained --- at $u=2$, or at one of the observed values lying in
$[0,2]$ --- and is a measurable function of $V$.

Put $Q:=\inf_{u\in[0,2]}R(h_{q,u})$ and $v:=\ln(12/\eta)/n$. The rules
$\{h_{q,u}:u\in[0,2]\}$ are thresholds of the single real feature
$a_q(Z)$, and so form a class of VC dimension one. An affine change of
variables carries the pairs $(a_q,u)$ into a common bounded interval, and
\cite[Lemma~3.3]{engelund2026optimal}, applied to the resulting
one-dimensional problem with labels rescaled to $\{-1,+1\}$, gives
\begin{equation}
    R(h_{q,\widehat u})\leq Q+C_T\bigl(\sqrt{Qv}+v\bigr)
       \leq(1+\alpha)Q+\frac{C v}\alpha,
    \label{eq:threshold-erm-rate}
\end{equation}
the second inequality by $C_T\sqrt{Qv}\leq\alpha Q+C_T^2v/(4\alpha)$.

\circled{2} \textbf{The best threshold is no worse than the truncated loss.}
It remains to compare $Q$ with $\Phi_\alpha(q)$, which we do by averaging
over thresholds rather than by exhibiting a good one. Fix an example $Z$.
By \eqref{eq:threshold-reduction}, the thresholds that err on $Z$ form an
interval: those at or below $a_q(Z)$ when $y=0$, and those strictly above
it when $y=1$. Measuring that interval inside $[\alpha,1-\alpha]$ gives,
in both cases and also when $\mathcal U(x)=\varnothing$,
\[
    \int_\alpha^{1-\alpha}\ell_{\mathcal U}(h_{q,u};Z)\,\mathrm du
       =\min\{\phi_\alpha(q;Z),\,1-2\alpha\}
       \leq\phi_\alpha(q;Z).
\]
Taking expectations over $Z$ and exchanging the order of integration,
\[
    (1-2\alpha)\,Q
    \leq\int_\alpha^{1-\alpha}R(h_{q,u})\,\mathrm du
    \leq\Phi_\alpha(q),
\]
since the average of $R(h_{q,u})$ over an interval of length $1-2\alpha$
is at least its infimum. Substituting $Q\leq\Phi_\alpha(q)/(1-2\alpha)$
into \eqref{eq:threshold-erm-rate} and using
$(1+\alpha)/(1-2\alpha)\leq1+4\alpha$ for $\alpha\leq1/8$ gives
\eqref{eq:truncated-threshold-selection}.
\end{proof}

We can now assemble the three steps.

\begin{proof}[Proof of Theorem~\ref{theorem:introduction_1}]
The degenerate cases were handled in Section~\ref{sec:agnostic-learner},
so assume $d\geq1$ and $m\geq6$. Throughout, $C$ denotes a universal
constant whose value may grow from one display to the next.

Fix $h\in\mathcal H$ and a deterministic $\alpha\in(0,1/8]$; both are
chosen at the end, and neither is available to the learner. With $k$ and
$b_t$ as in Algorithm~\ref{alg:learned-threshold}, write
$L:=\ln(6/\delta)$ and
\[
    X_t:=\phi_\alpha(b_t;Z_{t+1}),
    \qquad \rho_t:=\Phi_\alpha(b_t),
    \qquad k\leq t\leq2k-1;
\]
sums over $t$ below run over this range unless stated otherwise. Since
$0\leq\phi_\alpha\leq1$ we have $X_t\in[0,1]$, so
Lemma~\ref{lem:multiplicative-concentration} applies to these variables.

\circled{1} \textbf{Reading the prefix losses forwards.}
Let $\mathcal F_j:=\sigma(Z_1,\ldots,Z_{k+j})$ and $Y_j:=X_{k+j-1}$ for
$j=1,\ldots,k$. Then $Y_j$ is $\mathcal F_j$-measurable, and since
$b_{k+j-1}$ is determined by $Z_1,\ldots,Z_{k+j-1}$ while $Z_{k+j}$ is
independent of them,
$\mathbb E[Y_j\mid\mathcal F_{j-1}]=\rho_{k+j-1}$. In this direction each
loss is an unbiased estimate of its population value.

\circled{2} \textbf{Reading them backwards.}
Now reveal the sample in the opposite order, forgetting at each step
which of the first few examples was which. Every quantity in this step
depends on the sample only through $Z_1,\ldots,Z_{2k}$, so we may regard
these as the coordinates of the canonical product space
$(\mathcal X\times\mathcal Y)^{2k}$ carrying $\mathcal D^{\otimes2k}$.
For $s=k,\ldots,2k$, let $\mathfrak S_s$ act by permuting the first $s$
coordinates and fixing the rest, and let $\mathcal G_s$ be the
sub-$\sigma$-field of events invariant under $\mathfrak S_s$. Since
$\mathfrak S_k\subseteq\cdots\subseteq\mathfrak S_{2k}$, a larger group
leaves fewer events invariant, so
$\mathcal G_{2k}\subseteq\cdots\subseteq\mathcal G_k$. Moreover $X_t$ is
$\mathcal G_t$-measurable, because $b_t$ depends on $Z_1,\ldots,Z_t$
only through their multiset, while $Z_{t+1}$ is fixed by
$\mathfrak S_t$.

The group is finite and $\mathcal D^{\otimes2k}$ is invariant under it,
so conditioning is averaging over the group: for every integrable $f$,
\begin{equation}
    \mathbb E[f\mid\mathcal G_s](z)
       =\frac1{s!}\sum_{\pi\in\mathfrak S_s}f(\pi z).
    \label{eq:group-averaging}
\end{equation}
The right-hand side is invariant, hence $\mathcal G_s$-measurable, and
for every invariant event $A$ and every $\pi\in\mathfrak S_s$ we have
$\int_Af(\pi z)\,\mathrm dP=\int_{\pi A}f\,\mathrm dP=\int_Af\,\mathrm dP$,
which identifies it as the conditional expectation. This needs no
regular conditional distribution. Now apply
\eqref{eq:group-averaging} with $s=t+1$, collecting for each
$i\in[t+1]$ the $t!$ permutations that move $Z_i$ into coordinate
$t+1$; each of them leaves the first $t$ coordinates carrying the
multiset $(S_{\leq t+1})^{-i}$, so
\begin{align}
    \mu_t:=\mathbb E[X_t\mid\mathcal G_{t+1}]
      &=\frac1{t+1}\sum_{i=1}^{t+1}
           \phi_\alpha\bigl(b_{(S_{\leq t+1})^{-i}};Z_i\bigr)
      \notag\\
      &\leq\frac1{t+1}\left(
             r(S_{\leq t+1})+\frac{Cd}{\alpha}\right)
      \leq\frac1{t+1}\sum_{j=1}^{t+1}\ell_{\mathcal U}(h;Z_j)
              +\frac{Cd}{\alpha(t+1)},
    \label{eq:reverse-drift}
\end{align}
by Lemma~\ref{lem:truncated-loo} and the definition of $r$ as a minimum
over the class. To fit this into
Lemma~\ref{lem:multiplicative-concentration}, relabel: take
$\mathcal F_0':=\mathcal G_{2k}$, and
$\mathcal F_j':=\mathcal G_{2k-j}$ with $Y_j':=X_{2k-j}$ for
$j=1,\ldots,k$. Then $(\mathcal F_j')_{j=0}^k$ is an increasing
filtration, $Y_j'$ is $\mathcal F_j'$-measurable, and
$\mathbb E[Y_j'\mid\mathcal F_{j-1}']=\mu_{2k-j}$.

\circled{3} \textbf{Controlling the average truncated loss.}
Summing \eqref{eq:reverse-drift} over $t$ collects the losses of $h$ with
weights
\[
    c_j:=\sum_{t=k}^{2k-1}\frac{\mathds{1}\{j\leq t+1\}}{t+1},
    \qquad j\in[2k],
\]
which satisfy $0\leq c_j\leq\sum_t(t+1)^{-1}\leq1$ and
$\sum_{j=1}^{2k}c_j=k$. Write
$W:=\sum_{j=1}^{2k}c_j\,\ell_{\mathcal U}(h;Z_j)$ for the weighted loss
they produce; its summands are independent, lie in $[0,1]$, and have
means $c_jR(h)$, so $\mathbb E W=kR(h)$.

Apply Lemma~\ref{lem:multiplicative-concentration} with $\lambda=\alpha$
three times: its second inequality in the forward order of
\circled{1}, its first inequality in the reverse order of \circled{2},
and its first inequality to the summands of $W$. Each fails with
probability at most $\mathrm{e}^{-L}=\delta/6$, so by a union bound all
three hold at once with probability at least $1-\delta/2$:
\begin{align*}
    \sum\nolimits_t\rho_t \leq(1+\alpha)\bigl(\sum\nolimits_tX_t+L/\alpha\bigr), \qquad
    \sum\nolimits_tX_t, \leq(1+\alpha)\sum\nolimits_t\mu_t+L/\alpha,
\end{align*}
\begin{equation*}
        W \leq(1+\alpha)\,kR(h)+L/\alpha.
\end{equation*}
Summing \eqref{eq:reverse-drift} gives $\sum_t\mu_t\leq W+Cd/\alpha$.
Chaining the four inequalities, dividing by $k$, and using
$1+\alpha\leq2$ to absorb constants,
\[
    \frac1k\sum_t\rho_t\leq(1+\alpha)^3R(h)+C\frac{d+L}{\alpha k}.
\]
Finally, $\phi_\alpha$ is convex in its score argument, so Jensen's
inequality applied to $g_S=\frac1k\sum_tb_t$ gives
\begin{equation}
    \Phi_\alpha(g_S)
       \leq\frac1k\sum_t\rho_t
       \leq(1+\alpha)^3R(h)+C\frac{d+L}{\alpha k}.
    \label{eq:prefix-bound}
\end{equation}

\circled{4} \textbf{Selecting a threshold.}
Condition on $Z_1,\ldots,Z_{2k}$. The validation block $V$ is independent
of them, and $g_S$ is a fixed score with finite range, so
Lemma~\ref{lem:truncated-threshold-selection} applies with $q=g_S$,
$n=k$ and $\eta=\delta/2$. Combining it with \eqref{eq:prefix-bound} and
taking a union bound, with probability at least $1-\delta$,
\[
    R(\widehat h)
      \leq(1+4\alpha)(1+\alpha)^3R(h)+C\frac B\alpha,
    \qquad
    B:=\frac{d+\ln(24/\delta)}k .
\]
For $\alpha\leq1/8$ we have $(1+\alpha)^3\leq1+4\alpha$ and
$(1+4\alpha)^2\leq1+10\alpha$, so the leading factor is at most
$1+10\alpha$ and
\[
    R(\widehat h)\leq R(h)+C\left(\alpha R(h)+\frac B\alpha\right).
\]

\circled{5} \textbf{Optimizing the truncation level.}
The two error terms are balanced by
\[
    \alpha:=\frac18\sqrt{\frac B{B+R(h)}}\in(0,1/8],
\]
which is deterministic, as required. This choice gives
$\alpha R(h)\leq\sqrt{BR(h)}/8$ and
$B/\alpha=8\sqrt{B(B+R(h))}\leq8\bigl(B+\sqrt{BR(h)}\bigr)$. Since
$k\geq m/6$, and since $d\geq1$ makes
$\ln(24/\delta)\leq4\bigl(d+\log(1/\delta)\bigr)$, we obtain
\begin{equation}
    R(\widehat h)\leq R(h)+C\left(
       \sqrt{\frac{R(h)\bigl(d+\log(1/\delta)\bigr)}{m}}
       +\frac{d+\log(1/\delta)}m\right).
    \label{eq:bound-at-h}
\end{equation}

\circled{6} \textbf{Passing to the infimum.}
Inequality \eqref{eq:bound-at-h} holds for each fixed $h$, whereas the
theorem asks for $R^\star$, which need not be attained. Choose a
deterministic sequence $h_1,h_2,\ldots\in\mathcal H$ with
$R(h_j)\downarrow R^\star$, and let $E_j$ be the event
\eqref{eq:bound-at-h} for $h=h_j$. The right-hand side of
\eqref{eq:bound-at-h} is nondecreasing in $R(h)$, so
$E_1\supseteq E_2\supseteq\cdots$, and each $E_j$ has probability at
least $1-\delta$. By continuity of probability,
$\mathbb P[\bigcap_jE_j]\geq1-\delta$, and on $\bigcap_jE_j$ the bound
\eqref{eq:bound-at-h} holds with $R^\star$ in place of $R(h)$.
\end{proof}

\subsection{Minimal Measurability Requirements} \label{sec:minimal-measurability}

We use the standing measurability assumptions of Section~\ref{sec:preliminaries}. The fixed $\mathsf{RERM}$ has jointly measurable evaluation, so finite averaging makes $b_T$ and $g_S$ jointly measurable in the sample and query. The assumed joint measurability of their extrema over $\mathcal U(x)$ then makes the truncated losses and their population values measurable. Moreover, Lemma~\ref{lem:truncated-threshold-selection} shows that $\widehat u$ is selected from the endpoint $2$ and the finitely many validation values $a_{g_S}(Z)$ lying in $[0,2]$. These candidates and their empirical robust errors are measurable, so choosing the largest minimizer makes $\widehat u$ measurable in the full sample. Thus, $(S,x)\mapsto\widehat h(x)$ is jointly measurable, and \eqref{eq:threshold-reduction} gives the same conclusion for its robust loss, including when $\mathcal U(x)=\varnothing$.

Importantly, the supremum over $\mathcal H$ in Lemma~\ref{lem:truncated-loo} requires no additional class-level measurability assumption. There, the sample and witnesses are fixed before the random subset is drawn. On these witnesses, $\mathcal H$ induces only finitely many binary error patterns, so the Rademacher supremum is simply a maximum over a finite set, with randomness only in the finitely many signs. This remains true when the lemma is invoked in the proof of Theorem~\ref{theorem:introduction_1}, where it is applied to the random prefix $S_{\leq t+1}$, whose witnesses do vary with the sample. Lemma~\ref{lem:truncated-loo} is a deterministic inequality, valid for every fixed sample, so each such invocation is a pointwise application of it and never a statement about a random supremum. In particular, the argument requires no measurable choice of witnesses as the sample varies. For the same reason, $r$ needs no measurability argument of its own: $r(S_{\leq t+1})$ appears only as an intermediate term, immediately bounded by $\sum_{j}\ell_{\mathcal U}(h;Z_j)$ for the fixed $h$ under consideration.

This becomes especially transparent in the classical specialization $\mathcal U(x)=\{x\}$. The robust extrema then reduce to ordinary score evaluation, and jointly measurable evaluation of the fixed deterministic, permutation-invariant ERM suffices for the entire argument. Thus, the proof requires measurability only of the particular learner and losses that it constructs, rather than of random suprema over the whole class. In particular, no image-admissibility or class-level measurability assumption is needed. This provides a particularly direct result on the optimal PAC bounds under weak learner-level measurability assumptions. To the best of our knowledge, previous optimal PAC results for general $\operatorname{VC}$ classes have not explicitly isolated the result under such a weak measurability requirement. We emphasize, however, that this distinction is not an inherent limitation of earlier PAC approaches: with suitable measurable selections, several existing arguments can also be formulated under comparably weak assumptions. The advantage here is that the binomial-bagging proof makes this minimal dependence particularly explicit. For related measurability issues in classical PAC learning, see \cite{pestov2011pac}.

%% file: Main/Acknowledgments.tex
\section*{Acknowledgments}

Steve Hanneke acknowledges support by grant no.\ 2024243 from the United States - Israel Binational Science Foundation (BSF).

%% file: References.bib
@article{szegedy2013intriguing,
  title={Intriguing properties of neural networks},
  author={Szegedy, Christian and Zaremba, Wojciech and Sutskever, Ilya and Bruna, Joan and Erhan, Dumitru and Goodfellow, Ian and Fergus, Rob},
  journal={arXiv preprint arXiv:1312.6199},
  year={2013}
}

@inproceedings{dalvi2004adversarial,
  title     = {Adversarial Classification},
  author    = {Dalvi, Nilesh N. and Domingos, Pedro M. and Mausam and Sanghai, Sumit K. and Verma, Deepak},
  booktitle = {Proceedings of the Tenth {ACM} {SIGKDD} International Conference on Knowledge Discovery and Data Mining},
  series    = {{KDD} '04},
  pages     = {99--108},
  publisher = {Association for Computing Machinery},
  year      = {2004},
  doi       = {10.1145/1014052.1014066}
}

@inproceedings{lowd2005adversarial,
  title     = {Adversarial Learning},
  author    = {Lowd, Daniel and Meek, Christopher},
  booktitle = {Proceedings of the Eleventh {ACM} {SIGKDD} International Conference on Knowledge Discovery and Data Mining},
  series    = {{KDD} '05},
  pages     = {641--647},
  publisher = {Association for Computing Machinery},
  year      = {2005},
  doi       = {10.1145/1081870.1081950}
}

@inproceedings{globerson2006nightmare,
  title     = {Nightmare at Test Time: Robust Learning by Feature Deletion},
  author    = {Globerson, Amir and Roweis, Sam T.},
  booktitle = {Proceedings of the 23rd International Conference on Machine Learning},
  series    = {{ICML} '06},
  pages     = {353--360},
  publisher = {Association for Computing Machinery},
  year      = {2006},
  doi       = {10.1145/1143844.1143889}
}

@inproceedings{kolcz2009feature,
  title     = {Feature Weighting for Improved Classifier Robustness},
  author    = {Ko{\l}cz, Aleksander and Teo, Choon Hui},
  booktitle = {Proceedings of the Sixth Conference on Email and Anti-Spam ({CEAS})},
  pages     = {1--8},
  year      = {2009}
}

@inproceedings{goodfellow2015explaining,
  title     = {Explaining and Harnessing Adversarial Examples},
  author    = {Goodfellow, Ian J. and Shlens, Jonathon and Szegedy, Christian},
  booktitle = {International Conference on Learning Representations (ICLR)},
  year      = {2015}
}

@inproceedings{carlini2018audio,
  title     = {Audio adversarial examples: Targeted attacks on speech-to-text},
  author    = {Carlini, Nicholas and Wagner, David},
  booktitle = {2018 IEEE Security and Privacy Workshops (SPW)},
  pages     = {1--7},
  publisher = {IEEE},
  year      = {2018}
}

@inproceedings{carlini2016hidden,
  title     = {Hidden voice commands},
  author    = {Carlini, Nicholas and Mishra, Pratyush and Vaidya, Tavish and Zhang, Yuankai and Sherr, Micah and Shields, Clay and Wagner, David and Zhou, Wenchao},
  booktitle = {25th $\{$USENIX$\}$ Security Symposium ($\{$USENIX$\}$ Security 16)},
  pages     = {513--530},
  year      = {2016}
}

@inproceedings{lin2017tactics,
  title     = {Tactics of Adversarial Attack on Deep Reinforcement Learning Agents},
  author    = {Lin, Yen-Chen and Hong, Zhang-Wei and Liao, Yuan-Hong and Shih, Meng-Li and Liu, Ming-Yu and Sun, Min},
  booktitle = {Proceedings of the Twenty-Sixth International Joint Conference on Artificial Intelligence (IJCAI)},
  pages     = {3756--3762},
  year      = {2017},
  doi       = {10.24963/ijcai.2017/525}
}

@inproceedings{tretschk2018sequential,
  title     = {Sequential Attacks on Agents for Long-Term Adversarial Goals},
  author    = {Tretschk, Edgar and Oh, Seong Joon and Fritz, Mario},
  booktitle = {ACM Computer Science in Cars Symposium---Future Challenges in Artificial Intelligence and Security for Autonomous Vehicles (CSCS)},
  year      = {2018}
}

@misc{perez2022ignore,
  title         = {Ignore Previous Prompt: Attack Techniques for Language Models},
  author        = {Perez, F{\'a}bio and Ribeiro, Ian},
  year          = {2022},
  note          = {NeurIPS ML Safety Workshop},
  eprint        = {2211.09527},
  archiveprefix = {arXiv},
  primaryclass  = {cs.CL}
}

@inproceedings{greshake2023not,
  title     = {Not What You've Signed Up For: Compromising Real-World {LLM}-Integrated Applications with Indirect Prompt Injection},
  author    = {Greshake, Kai and Abdelnabi, Sahar and Mishra, Shailesh and Endres, Christoph and Holz, Thorsten and Fritz, Mario},
  booktitle = {Proceedings of the 16th {ACM} Workshop on Artificial Intelligence and Security},
  series    = {{AISec} '23},
  pages     = {79--90},
  publisher = {Association for Computing Machinery},
  year      = {2023}
}

@inproceedings{wei2023jailbroken,
  title     = {Jailbroken: How Does {LLM} Safety Training Fail?},
  author    = {Wei, Alexander and Haghtalab, Nika and Steinhardt, Jacob},
  booktitle = {Advances in Neural Information Processing Systems},
  volume    = {36},
  pages     = {80079--80110},
  year      = {2023}
}

@misc{zou2023universal,
  title         = {Universal and Transferable Adversarial Attacks on Aligned Language Models},
  author        = {Zou, Andy and Wang, Zifan and Carlini, Nicholas and Nasr, Milad and Kolter, J. Zico and Fredrikson, Matt},
  year          = {2023},
  eprint        = {2307.15043},
  archiveprefix = {arXiv},
  primaryclass  = {cs.CL}
}

@InProceedings{pmlr-v97-cohen19c,
    title = {Certified Adversarial Robustness via Randomized Smoothing},
    author = {Cohen, Jeremy and Rosenfeld, Elan and Kolter, Zico},
    booktitle = {Proceedings of the 36th International Conference on Machine Learning},
    pages = {1310--1320},
    year = {2019},
    editor = {Kamalika Chaudhuri and Ruslan Salakhutdinov}, volume = {97}, series = {Proceedings of Machine Learning Research}, month = {09--15 Jun}, publisher = {PMLR}
}

@article{carlini2016defensive,
  title={Defensive distillation is not robust to adversarial examples},
  author={Carlini, Nicholas and Wagner, David},
  journal={arXiv preprint arXiv:1607.04311},
  year={2016}
}

@article{carlini2017magnet,
  title={Magnet and" efficient defenses against adversarial attacks" are not robust to adversarial examples},
  author={Carlini, Nicholas and Wagner, David},
  journal={arXiv preprint arXiv:1711.08478},
  year={2017}
}

@misc{carlini2017evaluating,
      title={Towards Evaluating the Robustness of Neural Networks}, 
      author={Nicholas Carlini and David Wagner},
      year={2017},
      eprint={1608.04644},
      archivePrefix={arXiv},
      primaryClass={cs.CR}
}

@article{athalye2018robustness,
  title={On the robustness of the cvpr 2018 white-box adversarial example defenses},
  author={Athalye, Anish and Carlini, Nicholas},
  journal={arXiv preprint arXiv:1804.03286},
  year={2018}
}

@inproceedings{athalye2018obfuscated,
  title={Obfuscated gradients give a false sense of security: Circumventing defenses to adversarial examples},
  author={Athalye, Anish and Carlini, Nicholas and Wagner, David},
  booktitle={International Conference on Machine Learning},
  pages={274--283},
  year={2018},
  organization={PMLR}
}

@article{carlini2019ami,
  title={Is ami (attacks meet interpretability) robust to adversarial examples?},
  author={Carlini, Nicholas},
  journal={arXiv preprint arXiv:1902.02322},
  year={2019}
}

@article{carlini2019evaluating,
  title={On evaluating adversarial robustness},
  author={Carlini, Nicholas and Athalye, Anish and Papernot, Nicolas and Brendel, Wieland and Rauber, Jonas and Tsipras, Dimitris and Goodfellow, Ian and Madry, Aleksander and Kurakin, Alexey},
  journal={arXiv preprint arXiv:1902.06705},
  year={2019}
}

@misc{tramer2020adaptive,
      title={On Adaptive Attacks to Adversarial Example Defenses}, 
      author={Florian Tramer and Nicholas Carlini and Wieland Brendel and Aleksander Madry},
      year={2020},
      eprint={2002.08347},
      archivePrefix={arXiv},
      primaryClass={cs.LG}
}

@misc{kaur2019perceptuallyaligned,
      title={Are Perceptually-Aligned Gradients a General Property of Robust Classifiers?}, 
      author={Simran Kaur and Jeremy Cohen and Zachary C. Lipton},
      year={2019},
      eprint={1910.08640},
      archivePrefix={arXiv},
      primaryClass={cs.LG}
}

@misc{santurkar2019image,
      title={Image Synthesis with a Single (Robust) Classifier}, 
      author={Shibani Santurkar and Dimitris Tsipras and Brandon Tran and Andrew Ilyas and Logan Engstrom and Aleksander Madry},
      year={2019},
      eprint={1906.09453},
      archivePrefix={arXiv},
      primaryClass={cs.CV}
}

@article{engstrom2019adversarial,
  title={Adversarial robustness as a prior for learned representations},
  author={Engstrom, Logan and Ilyas, Andrew and Santurkar, Shibani and Tsipras, Dimitris and Tran, Brandon and Madry, Aleksander},
  journal={arXiv preprint arXiv:1906.00945},
  year={2019}
}

@misc{zhang2019interpreting,
      title={Interpreting Adversarially Trained Convolutional Neural Networks}, 
      author={Tianyuan Zhang and Zhanxing Zhu},
      year={2019},
      eprint={1905.09797},
      archivePrefix={arXiv},
      primaryClass={cs.LG}
}

@misc{tsipras2019robustness,
      title={Robustness May Be at Odds with Accuracy}, 
      author={Dimitris Tsipras and Shibani Santurkar and Logan Engstrom and Alexander Turner and Aleksander Madry},
      year={2019},
      eprint={1805.12152},
      archivePrefix={arXiv},
      primaryClass={stat.ML}
}

@article{salehi2020arae,
  title={Arae: Adversarially robust training of autoencoders improves novelty detection},
  author={Salehi, Mohammadreza and Arya, Atrin and Pajoum, Barbod and Otoofi, Mohammad and Shaeiri, Amirreza and Rohban, Mohammad Hossein and Rabiee, Hamid R},
  journal={arXiv preprint arXiv:2003.05669},
  year={2020}
}

@misc{salman2020adversarially,
      title={Do Adversarially Robust ImageNet Models Transfer Better?}, 
      author={Hadi Salman and Andrew Ilyas and Logan Engstrom and Ashish Kapoor and Aleksander Madry},
      year={2020},
      eprint={2007.08489},
      archivePrefix={arXiv},
      primaryClass={cs.CV}
}

@misc{kurakin2016adversarial,
  title={Adversarial examples in the physical world},
  author={Kurakin, Alexey and Goodfellow, Ian and Bengio, Samy and others},
  year={2016}
}

@article{madry2017towards,
  title={Towards deep learning models resistant to adversarial attacks},
  author={Madry, Aleksander and Makelov, Aleksandar and Schmidt, Ludwig and Tsipras, Dimitris and Vladu, Adrian},
  journal={arXiv preprint arXiv:1706.06083},
  year={2017}
}

@misc{gowal2020uncovering,
      title={Uncovering the Limits of Adversarial Training against Norm-Bounded Adversarial Examples}, 
      author={Sven Gowal and Chongli Qin and Jonathan Uesato and Timothy Mann and Pushmeet Kohli},
      year={2020},
      eprint={2010.03593},
      archivePrefix={arXiv},
      primaryClass={stat.ML}
}

@misc{xie2020smooth,
      title={Smooth Adversarial Training}, 
      author={Cihang Xie and Mingxing Tan and Boqing Gong and Alan Yuille and Quoc V. Le},
      year={2020},
      eprint={2006.14536},
      archivePrefix={arXiv},
      primaryClass={cs.LG}
}

@inproceedings{rice2020overfitting,
  title={Overfitting in adversarially robust deep learning},
  author={Rice, Leslie and Wong, Eric and Kolter, Zico},
  booktitle={International Conference on Machine Learning},
  pages={8093--8104},
  year={2020},
  organization={PMLR}
}

@article{schmidt2018adversarially,
  title={Adversarially robust generalization requires more data},
  author={Schmidt, Ludwig and Santurkar, Shibani and Tsipras, Dimitris and Talwar, Kunal and Madry, Aleksander},
  journal={Advances in neural information processing systems},
  volume={31},
  year={2018}
}

@misc{vapnik1974theory,
  title={Theory of pattern recognition},
  author={Vapnik, Vladimir and Chervonenkis, Alexey},
  year={1974},
  publisher={Nauka, Moscow}
}

@article{sauer1972density,
  title={On the density of families of sets},
  author={Sauer, Norbert},
  journal={Journal of Combinatorial Theory, Series A},
  volume={13},
  number={1},
  pages={145--147},
  year={1972},
  publisher={Elsevier}
}

@article{shelah1972combinatorial,
  title={A combinatorial problem; stability and order for models and theories in infinitary languages},
  author={Shelah, Saharon},
  journal={Pacific Journal of Mathematics},
  volume={41},
  number={1},
  pages={247--261},
  year={1972},
  publisher={Mathematical Sciences Publishers}
}

@article{valiant1984theory,
  title={A theory of the learnable},
  author={Valiant, Leslie G},
  journal={Communications of the ACM},
  volume={27},
  number={11},
  pages={1134--1142},
  year={1984},
  publisher={ACM New York, NY, USA}
}

@article{assouad1983densite,
  author  = {Assouad, Patrick},
  title   = {Densit{\'e} et dimension},
  journal = {Annales de l'Institut Fourier},
  volume  = {33},
  number  = {3},
  pages   = {233--282},
  year    = {1983},
  doi     = {10.5802/aif.938}
}

@article{kleer2023primal,
  author  = {Kleer, Pieter and Simon, Hans},
  title   = {Primal and dual combinatorial dimensions},
  journal = {Discrete Applied Mathematics},
  volume  = {327},
  pages   = {185--196},
  year    = {2023},
  doi     = {10.1016/j.dam.2022.11.010}
}

@book{devroye2001combinatorial,
  title     = {Combinatorial Methods in Density Estimation},
  author    = {Devroye, Luc and Lugosi, G{\'a}bor},
  series    = {Springer Series in Statistics},
  publisher = {Springer},
  address   = {New York},
  year      = {2001},
  doi       = {10.1007/978-1-4613-0125-7}
}

@book{devroye1996probabilistic,
  title     = {A Probabilistic Theory of Pattern Recognition},
  author    = {Devroye, Luc and Gy{\"o}rfi, L{\'a}szl{\'o} and Lugosi, G{\'a}bor},
  year      = {1996},
  publisher = {Springer},
  address   = {New York}
}

@article{hanneke2016optimal,
  title={The optimal sample complexity of PAC learning},
  author={Hanneke, Steve},
  journal={Journal of Machine Learning Research},
  volume={17},
  number={38},
  pages={1--15},
  year={2016}
}

@article{moran2016sample,
  title={Sample compression schemes for VC classes},
  author={Moran, Shay and Yehudayoff, Amir},
  journal={Journal of the ACM (JACM)},
  volume={63},
  number={3},
  pages={1--10},
  year={2016},
  publisher={ACM New York, NY, USA}
}

@article{david2016supervised,
  title={Supervised learning through the lens of compression},
  author={David, Ofir and Moran, Shay and Yehudayoff, Amir},
  journal={Advances in Neural Information Processing Systems},
  volume={29},
  year={2016}
}

@inproceedings{bousquet2020proper,
  title={Proper learning, Helly number, and an optimal SVM bound},
  author={Bousquet, Olivier and Hanneke, Steve and Moran, Shay and Zhivotovskiy, Nikita},
  booktitle={Conference on Learning Theory},
  pages={582--609},
  year={2020},
  organization={PMLR}
}

@article{hanneke2021universal,
  title   = {Universal {Bayes} consistency in metric spaces},
  author  = {Hanneke, Steve and Kontorovich, Aryeh and
             Sabato, Sivan and Weiss, Roi},
  journal = {The Annals of Statistics},
  volume  = {49},
  number  = {4},
  pages   = {2129--2150},
  year    = {2021},
  doi     = {10.1214/20-AOS2029}
}

@InProceedings{pmlr-v178-hopkins22a,
  title = 	 {Realizable Learning is All You Need},
  author =       {Hopkins, Max and Kane, Daniel M. and Lovett, Shachar and Mahajan, Gaurav},
  booktitle = 	 {Proceedings of Thirty Fifth Conference on Learning Theory},
  pages = 	 {3015--3069},
  year = 	 {2022},
  editor = 	 {Loh, Po-Ling and Raginsky, Maxim},
  volume = 	 {178},
  series = 	 {Proceedings of Machine Learning Research},
  month = 	 {02--05 Jul},
  publisher =    {PMLR}
}

@inproceedings{larsen2023bagging,
  title={Bagging is an optimal PAC learner},
  author={Larsen, Kasper Green},
  booktitle={The Thirty Sixth Annual Conference on Learning Theory},
  pages={450--468},
  year={2023},
  organization={PMLR}
}

@inproceedings{aden2023optimal,
  title={Optimal pac bounds without uniform convergence},
  author={Aden-Ali, Ishaq and Cherapanamjeri, Yeshwanth and Shetty, Abhishek and Zhivotovskiy, Nikita},
  booktitle={2023 IEEE 64th Annual Symposium on Foundations of Computer Science (FOCS)},
  pages={1203--1223},
  year={2023},
  organization={IEEE}
}

@inproceedings{hanneke2024revisiting,
  title={Revisiting agnostic PAC learning},
  author={Hanneke, Steve and Larsen, Kasper Green and Zhivotovskiy, Nikita},
  booktitle={2024 IEEE 65th Annual Symposium on Foundations of Computer Science (FOCS)},
  pages={1968--1982},
  year={2024},
  organization={IEEE}
}

@article{asilis2026agnostic,
  title={On Agnostic PAC Learning in the Small Error Regime},
  author={Asilis, Julian and M{\o}ller H{\o}gsgaard, Mikael and Velegkas, Grigoris},
  journal={Advances in Neural Information Processing Systems},
  volume={38},
  pages={123346--123388},
  year={2026}
}

@article{rawal2026majority,
  title={Majority-of-Three is Optimal},
  author={Rawal, Divit and Zhivotovskiy, Nikita},
  journal={arXiv preprint arXiv:2606.13614},
  year={2026}
}

@article{engelund2026optimal,
  title={An Optimal Agnostic {PAC} Algorithm},
  author={Engelund Mathiasen, Markus and Qian, Jian and Zhivotovskiy, Nikita},
  journal={arXiv preprint arXiv:2608.06363},
  year={2026}
}

@inproceedings{karbasi2024impossibility,
  title={The impossibility of parallelizing boosting},
  author={Karbasi, Amin and Larsen, Kasper Green},
  booktitle={International Conference on Algorithmic Learning Theory},
  pages={635--653},
  year={2024},
  organization={PMLR}
}

@inproceedings{hanneke2025representation,
  title={Representation Preserving Multiclass Agnostic to Realizable Reduction.},
  author={Hanneke, Steve and Meng, Qinglin and Shaeiri, Amirreza},
  booktitle={ICML},
  year={2025}
}

@inproceedings{pestov2011pac,
  title={PAC learnability versus VC dimension: a footnote to a basic result of statistical learning},
  author={Pestov, Vladimir},
  booktitle={The 2011 International Joint Conference on Neural Networks},
  pages={1141--1145},
  year={2011},
  organization={IEEE}
}

@InProceedings{pmlr-v99-montasser19a,
  title = 	 {VC Classes are Adversarially Robustly Learnable, but Only Improperly},
  author =       {Montasser, Omar and Hanneke, Steve and Srebro, Nathan},
  booktitle = 	 {Proceedings of the Thirty-Second Conference on Learning Theory},
  pages = 	 {2512--2530},
  year = 	 {2019},
  editor = 	 {Beygelzimer, Alina and Hsu, Daniel},
  volume = 	 {99},
  series = 	 {Proceedings of Machine Learning Research},
  month = 	 {25--28 Jun},
  publisher =    {PMLR}
}

@InProceedings{pmlr-v134-montasser21a,
  title = 	 {Adversarially Robust Learning with Unknown Perturbation Sets},
  author =       {Montasser, Omar and Hanneke, Steve and Srebro, Nathan},
  booktitle = 	 {Proceedings of Thirty Fourth Conference on Learning Theory},
  pages = 	 {3452--3482},
  year = 	 {2021},
  editor = 	 {Belkin, Mikhail and Kpotufe, Samory},
  volume = 	 {134},
  series = 	 {Proceedings of Machine Learning Research},
  month = 	 {15--19 Aug},
  publisher =    {PMLR}
}

@article{montasser2020reducing,
  title={Reducing adversarially robust learning to non-robust pac learning},
  author={Montasser, Omar and Hanneke, Steve and Srebro, Nati},
  journal={Advances in Neural Information Processing Systems},
  volume={33},
  pages={14626--14637},
  year={2020}
}

@InProceedings{pmlr-v151-montasser22a,
  title = 	 { Transductive Robust Learning Guarantees },
  author =       {Montasser, Omar and Hanneke, Steve and Srebro, Nathan},
  booktitle = 	 {Proceedings of The 25th International Conference on Artificial Intelligence and Statistics},
  pages = 	 {11461--11471},
  year = 	 {2022},
  editor = 	 {Camps-Valls, Gustau and Ruiz, Francisco J. R. and Valera, Isabel},
  volume = 	 {151},
  series = 	 {Proceedings of Machine Learning Research},
  month = 	 {28--30 Mar},
  publisher =    {PMLR}
}

@article{montasser2022adversarially,
  title={Adversarially robust learning: A generic minimax optimal learner and characterization},
  author={Montasser, Omar and Hanneke, Steve and Srebro, Nati},
  journal={Advances in Neural Information Processing Systems},
  volume={35},
  pages={37458--37470},
  year={2022}
}

@misc{montasser2026baggingrobustlylearnsvc,
      title={Bagging Robustly Learns VC Classes with Linear Sample Complexity}, 
      author={Omar Montasser},
      year={2026},
      eprint={2608.13514},
      archivePrefix={arXiv},
      primaryClass={stat.ML}
}

@inproceedings{attias2019improved,
  title={Improved generalization bounds for robust learning},
  author={Attias, Idan and Kontorovich, Aryeh and Mansour, Yishay},
  booktitle={Algorithmic Learning Theory},
  pages={162--183},
  year={2019},
  organization={PMLR}
}

@article{attias2022characterization,
  title={A characterization of semi-supervised adversarially robust pac learnability},
  author={Attias, Idan and Hanneke, Steve and Mansour, Yishay},
  journal={Advances in Neural Information Processing Systems},
  volume={35},
  pages={23646--23659},
  year={2022}
}

@inproceedings{attias2023adversarially,
  title={Adversarially robust pac learnability of real-valued functions},
  author={Attias, Idan and Hanneke, Steve},
  booktitle={International Conference on Machine Learning},
  pages={1172--1199},
  year={2023},
  organization={PMLR}
}

@InProceedings{pmlr-v272-attias25a,
  title = 	 {Sample Compression Scheme Reductions},
  author =       {Attias, Idan and Hanneke, Steve and Ramaswami, Arvind},
  booktitle = 	 {Proceedings of The 36th International Conference on Algorithmic Learning Theory},
  pages = 	 {134--162},
  year = 	 {2025},
  editor = 	 {Kamath, Gautam and Loh, Po-Ling},
  volume = 	 {272},
  series = 	 {Proceedings of Machine Learning Research},
  month = 	 {24--27 Feb},
  publisher =    {PMLR}
}

@article{breiman1996bagging,
  title={Bagging predictors},
  author={Breiman, Leo},
  journal={Machine learning},
  volume={24},
  number={2},
  pages={123--140},
  year={1996},
  publisher={Springer}
}

@article{buhlmann2002analyzing,
  title={Analyzing bagging},
  author={B{\"u}hlmann, Peter and Yu, Bin},
  journal={The annals of Statistics},
  volume={30},
  number={4},
  pages={927--961},
  year={2002},
  publisher={Institute of Mathematical Statistics}
}

@inproceedings{oza2001online,
  title={Online bagging and boosting},
  author={Oza, Nikunj C and Russell, Stuart J},
  booktitle={International workshop on artificial intelligence and statistics},
  pages={229--236},
  year={2001},
  organization={PMLR}
}

@article{wu2025ensemble,
  title={Ensemble linear interpolators: The role of ensembling},
  author={Wu, Mingqi and Sun, Qiang},
  journal={SIAM Journal on Mathematics of Data Science},
  volume={7},
  number={2},
  pages={438--467},
  year={2025},
  publisher={SIAM}
}

@inproceedings{ashtiani2023adversarially,
  title={Adversarially robust learning with tolerance},
  author={Ashtiani, Hassan and Pathak, Vinayak and Urner, Ruth},
  booktitle={International Conference on Algorithmic Learning Theory},
  pages={115--135},
  year={2023},
  organization={PMLR}
}

@inproceedings{bhattacharjee2023robust,
  title={Robust empirical risk minimization with tolerance},
  author={Bhattacharjee, Robi and Hopkins, Max and Kumar, Akash and Yu, Hantao and Chaudhuri, Kamalika},
  booktitle={International Conference on Algorithmic Learning Theory},
  pages={182--203},
  year={2023},
  organization={PMLR}
}

@article{raman2023proper,
  title={On proper learnability between average-and worst-case robustness},
  author={Raman, Vinod and Subedi, Unique and Tewari, Ambuj},
  journal={Advances in Neural Information Processing Systems},
  volume={36},
  pages={13115--13126},
  year={2023}
}

@article{ashtiani2025simplifying,
  title={Simplifying Adversarially Robust PAC Learning with Tolerance},
  author={Ashtiani, Hassan and Pathak, Vinayak and Urner, Ruth},
  journal={arXiv preprint arXiv:2502.07232},
  year={2025}
}

@inproceedings{bubeck2021law,
  title     = {A Law of Robustness for Two-Layers Neural Networks},
  author    = {Bubeck, Sebastien and Li, Yuanzhi and Nagaraj, Dheeraj M.},
  booktitle = {Proceedings of the Thirty-Fourth Conference on Learning Theory},
  editor    = {Belkin, Mikhail and Kpotufe, Samory},
  series    = {Proceedings of Machine Learning Research},
  volume    = {134},
  pages     = {804--820},
  year      = {2021},
  publisher = {PMLR}
}

@inproceedings{bubeck2021universal,
  title     = {A Universal Law of Robustness via Isoperimetry},
  author    = {Bubeck, S{\'e}bastien and Sellke, Mark},
  booktitle = {Advances in Neural Information Processing Systems},
  editor    = {Ranzato, Marc'Aurelio and Beygelzimer, Alina and
               Dauphin, Yann and Liang, Percy and Vaughan, Jennifer Wortman},
  volume    = {34},
  pages     = {28811--28822},
  year      = {2021},
  publisher = {Curran Associates, Inc.}
}

@misc{more2026order,
  title         = {Does Order Matter: Connecting the Law of Robustness to Robust Generalization},
  author        = {More, Mihir and Das, Aritra and Ponde, Jaee and Mandal, Himadri and Varadarajan, Vishnu and Gupta, Debayan},
  year          = {2026},
  eprint        = {2602.20971},
  archivePrefix = {arXiv},
  primaryClass  = {cs.LG},
  doi           = {10.48550/arXiv.2602.20971}
}

@inproceedings{dern2025theoretical,
  title     = {Theoretical Limitations of Ensembles in the Age of
               Overparameterization},
  author    = {Dern, Niclas and Cunningham, John Patrick and Pleiss, Geoff},
  booktitle = {Proceedings of the 42nd International Conference on Machine Learning},
  editor    = {Singh, Aarti and Fazel, Maryam and Hsu, Daniel and
               Lacoste-Julien, Simon and Berkenkamp, Felix and Maharaj, Tegan and
               Wagstaff, Kiri and Zhu, Jerry},
  series    = {Proceedings of Machine Learning Research},
  volume    = {267},
  pages     = {13369--13401},
  year      = {2025},
  publisher = {PMLR}
}
